\documentclass[letterpaper]{article}
\usepackage{aaai19}  
\usepackage{times}  
\usepackage{helvet}  
\usepackage{courier}  
\usepackage{url}  
\usepackage{graphicx}  

\usepackage{times}  
\usepackage{helvet}  
\usepackage{courier}  
\usepackage{url}  
\usepackage{caption}
\usepackage{epstopdf} 
\usepackage{amssymb}

\usepackage{amsmath}
\usepackage{amsfonts}
\usepackage[english]{babel}
\usepackage[bottom]{footmisc}
\usepackage{wrapfig}
\usepackage[noend]{algpseudocode}
\usepackage{algorithm}
\usepackage{algpseudocode}
\usepackage{amsthm}
\usepackage{siunitx}
\usepackage{array}
\usepackage{float}
\usepackage{multirow}

\usepackage{amsmath}
\usepackage[font=footnotesize,caption=true]{subfig}
\usepackage{adjustbox}

\usepackage[colorinlistoftodos,prependcaption,textsize=tiny]{todonotes}
\newtheorem{theorem}{Theorem}

\newtheorem{myprobl}{Problem}

\newtheorem{myexam}{Example}

\newcommand{\GP}{\mathbf{z}}

\newcommand{\tss}[0]{{N}}

\begin{document}
%
\title{Robustness Guarantees for Bayesian Inference with Gaussian Processes 
}
\author{Luca Cardelli, Marta Kwiatkowska, Luca Laurenti, Andrea Patane\\
University of Oxford\\
Microsoft Research Cambridge
}
\maketitle
\begin{abstract}
Bayesian inference and Gaussian processes  are widely used 
in
applications ranging from robotics and control to biological systems. 
Many of these applications are safety-critical and require a 
characterization of the uncertainty associated with the learning model 
and formal guarantees on its predictions. In this paper we define a 
robustness measure for Bayesian inference against input perturbations, 
given by the probability that, for a test point and a compact set in the 
input space containing the test point, the prediction of the learning 
model will remain $\delta-$close for all the points in the set, for 
$\delta>0.$ Such measures can be used to provide formal guarantees for 
the absence of adversarial examples.
By employing the theory of Gaussian processes, we derive tight upper 
bounds on the resulting robustness  
by utilising the Borell-TIS 
inequality, and propose algorithms for their computation.
We evaluate our techniques on two examples, a GP regression problem and 
a fully-connected deep neural network, where we rely on weak convergence 
to GPs to study adversarial examples on the MNIST dataset\footnote{Code is available at https://github.com/andreapatane/checkGP.}.

\end{abstract}

\section{Introduction}

The widespread deployment of 
machine learning models,
coupled with the discovery of their fragility against carefully crafted manipulation of training and/or test samples \cite{biggio2017wild,grosse2017statistical,szegedy2013intriguing}, calls for safe approaches to AI to enable their use in 
safety-critical applications, as argued, e.g., in \cite{seshia2016towards,dreossi2017compositional}. 
Bayesian techniques, in particular, provide a principled way of combining a-priori information into the training process, so as to obtain an a-posteriori distribution on test data, which also takes into account the uncertainty in the learning process. 
Recent advances in Bayesian learning include adversarial attacks 
\cite{grosse2017wrong} and methods to compute pointwise uncertainty estimates in Bayesian deep learning 
\cite{gal2016dropout}.
However, much of the work on formal guarantees for machine learning models 
has focused on non-Bayesian models, such as deep neural networks (NNs) \cite{huang2017safety,hein2017formal} 
and, to the best of our knowledge, 
there is no work directed at providing formal guarantees for the absence of
adversarial local input perturbations in Bayesian prediction settings. 

Gaussian processes (GPs) are a class of stochastic processes that are, due to their many favourable properties, widely employed for Bayesian learning \cite{rasmussen2004gaussian},
with applications spanning robotics, control systems and biological processes \cite{sadigh2015safe,laurenti2017reachability,bortolussi2018central}. 
Further, driven by pioneering work that first recognized the convergence of fully-connected NNs to GPs in the limit of infinitely many neurons \cite{neal2012bayesian}, 
GPs have been used recently as a model to characterize the behaviour of NNs in terms of convergence analysis \cite{matthews2018gaussian}, approximated Bayesian inference \cite{lee2017deep} and training algorithms \cite{chouza2018gradient}.

In this paper we compute formal local robustness guarantees for Bayesian inference with GP priors. The resulting guarantees are probabilistic, as they take into account the uncertainty intrinsic in the Bayesian learning process and explicitly work with the a-posteriori output distribution of the GP.
More specifically, given a GP model trained on a given data set, a test input point and a neighborhood around the latter, we are interested in computing the probability that there exists a point in the neighbourhood such that the prediction of the GP on the latter differs from the initial test input point by at least a given threshold. This 
implicitly gives guarantees on the absence of adversarial examples, that is, input samples that  trick a machine learning model into performing wrong predictions.

Unfortunately, computing such a probability is far from trivial. In fact, 
given a compact set $T\subseteq \mathbb{R}^m,m>0,$ and $x^* \in T,$ the above measure reduces to computing the probability that there exists a function $f$ sampled from the GP such that there exists $x \in T$ for which $||f(x^*)-f(x) ||> \delta$, where $\delta>0$ and $|| \cdot ||$ is a metric norm. Since the set $T$ is composed of an infinite number of points, computing such a measure for general stochastic processes is extremely challenging. 
However, for GPs we can obtain tight upper bounds on the above probability by making use of inequalities developed in the theory of GPs, such as the \emph{Borell-TIS inequality} \cite{adler2009random} and the \emph{Dudley's Entropy Integral} \cite{dudley1967sizes}. 
To do this, we need to obtain lower and upper bounds on the extrema of the a-posteriori GP mean and variance functions on neighborhoods of a given test point.
We obtain these bounds by constructing lower and upper approximations for the GP kernel as a function of the test point, which are then propagated through the GP inference formulas. Then, safe approximations for these values are obtained  by posing a series of optimization problems that can be solved either analytically or by standard {quadratic} convex optimization techniques.
We illustrate the above framework with explicit algorithmic techniques for GPs built with squared-exponential and ReLU kernel.

Finally, we apply the methods presented here to characterize the robustness of GPs with ReLU kernel trained on a subset of images included in the MNIST dataset.
Relying on the weak convergence between fully-connected NNs with ReLU activation functions and the corresponding GPs with ReLU kernel, we analyze the behaviour of 
such networks on adversarial images in a Bayesian setting.
We use SIFT \cite{lowe2004distinctive} to focus on important patches of the image, and perform feature-level safety analyses of test points included in the dataset.
We apply the proposed methods to evaluate the resilience of features against (generic) adversarial perturbations bounded in norm, and discuss how this is affected by stronger perturbations and different misclassification thresholds.
We perform a parametric optimization analysis of maximum prediction variance around specific test points in an effort to characterize active defenses against adversarial examples that rely on variance thresholding. 
In the examples we studied, we have consistently observed that, while an increased number of training samples may significantly help detect adversarial examples by means of prediction uncertainty, the process may be undermined by more complex architectures.  

In summary, the paper makes the following main contributions:
\begin{itemize}
    \item We provide tight upper bounds on the probability that the prediction of a Gaussian process remains close to a given test point in a neighbourhood, which can be used to quantify local robustness against adversarial examples. 
    \item 
    We develop algorithmic methods for the computation of extrema of GP mean and variance over a compact set.
    \item Relying on convergence between fully-connected NNs and GPs, we apply the developed methods to provide feature-level analysis of the behaviour of the former on 
    the MNIST dataset.
\end{itemize}

\subsubsection{Why probabilistic local robustness guarantees?}
Our results provide formal local robustness guarantees in the sense that the resulting bounds are sound with respect to a neighbourhood of an input, and 
numerical methods 
have not been used. 
This enables certification for Bayesian methods that is necessary in safety-critical applications, 
and is in contrast with many existing pointwise approaches to detect adversarial examples in Bayesian models, which are generally based on heuristics, such as to reject test points with high uncertainty \cite{li2017dropout,feinman2017detecting}.
We illustrate the intuition with the following simple example.

\begin{myexam}\label{IntroExample}
Let $(\GP(x),x \in \mathbb{N})$ be a zero-mean stochastic process with values in $\mathbb{R}$. Consider the following widely used definition of safety for a set $T=[x_{1},...,x_{10}]$
\begin{align}
\label{SafetyDefinitionIntroExample}
    P_{safe}(\GP,T,\delta)=\text{Prob}( \forall x \in T,\,   \GP(x) < \delta ), 
\end{align} 
where $\delta\in \mathbb{R}$ is a given threshold.
Assume that, for all $x_i,x_j \in T$, $\GP(x_i)$ and $\GP(x_j)$ are independently and equally distributed random variables such that for each $x \in T$ we have 
$ \text{Prob}( \GP(x) < \delta )=0.85. $
Then, if we compute the above property we obtain
$$  P_{safe}(\GP,T,\delta)=0.85^{10}\approx 0.197 .$$
Thus, even though at each point $\GP(x)$ has relatively high probability of being safe, $P_{safe}(\GP,T,\delta)$ is still small. This is because safety, as defined in Eqn \ref{SafetyDefinitionIntroExample}, depends on a set of points, and this must be accounted for to give robustness guarantees for a given stochastic model. Note that, to simplify, we used a discrtete set $T$, but the same reasoning remains valid even if $T\subseteq \mathbb{R}^m,m>0$, as in this paper.
\end{myexam}

\subsection{Related Work}
Existing formal approaches for machine learning models mostly focus on computing non-probabilistic local safety guarantees \cite{raghunathan2018certified,huang2017safety,ruan2018reachability} and generally neglect the uncertainty {of} the learning process, which 
is intrinsic in a Bayesian model.
Recently, empirical methods to detect adversarial examples for Bayesian NNs that utilise pointwise uncertainty have been introduced \cite{li2017dropout,feinman2017detecting}.
However, these approaches can be fooled by attacks that generate adversarial examples with small uncertainty as shown in \cite{carlini2017adversarial}. 
Unfortunately, obtaining formal guarantees for Bayesian NNs is challenging since their posterior distribution, which can be obtained in closed form for GPs, is generally analytically intractable \cite{gal2016dropout}.  
In \cite{grosse2017wrong} attacks for Bayesian inference with Gaussian processes based on local perturbations of the mean have been presented. 

Notions of safety for Gaussian processes have been recently studied in the context of system design for stochastic models  (see, e.g. \cite{wachi2018safe,bartocci2015system,sadigh2015safe,sui2015safe}). In \cite{sadigh2015safe}, the authors synthesize safe controllers against \emph{Probabilistic Signal Temporal Logic (PrSTL)} specifications, which suffer from the 
issue illustrated in Example \ref{IntroExample}.  
Another related approach is that in  \cite{sui2015safe}, where the authors build on 
\cite{srinivas2010gaussian} and introduce SAFEOPT, a Bayesian optimization algorithm that additionally guarantees that, for the optimized parameters, with high probability the resulting objective function (sampled from a GP) is greater than a threshold. However, they do not give guarantees against perturbation of the synthesized parameters. {For instance, their method cannot guarantee that the resulting behaviour will still be safe and close to the optimal value if parameters corrupted by noise are applied.} 
Our approach 
allows one to quantify such a probability.
 We should also stress that, while it is often the case that the guarantees provided by existing algorithms are statistical (i.e., given in terms of confidence intervals), the bounds presented in this paper are probabilistic.



 \section{Problem Formulation}


We consider a Gaussian process $\big(\GP(x),x \in \mathbb{R}^m,m>0\big) $ with values in $\mathbb{R}^n,n>0$ and  with a Gaussian probability measure $P$ such that, for any $x_1,x_2,...,x_k \in \mathbb{R}^m$, $P(\GP(x_1),\GP(x_2),...,\GP(x_k))$ is a multivariate normal distribution\footnote{In this paper we assume $\GP$ is a separable stochastic process. This is a standard and common assumption  \cite{adler2009random}. The separability of $\GP$ guarantees that Problem \ref{BoundedVariationSingleCOmponentDefiniton} and \ref{BoundedVariationDefinition} are measurable.}. We consider Bayesian inference for $\GP$. That is, as illustrated in detail in the next section, given a dataset $\mathcal{D}=\{\GP(x_i)=y_i, i \in \{1,...,\tss \} \}$ {of $\vert \mathcal{D} \vert :=\tss$} samples,  we consider the process
$$ \GP(x) | \mathcal{D}, \, x \in \mathbb{R}^m,   $$
which represents the conditional distribution of $\GP$ given the set of observations in $\mathcal{D}$.
The first problem we examine is Problem \ref{BoundedVariationSingleCOmponentDefiniton}, where we want to compute the probability that local perturbations of a given test point result
in predictions that remain close to the original. 
\begin{myprobl}\label{BoundedVariationSingleCOmponentDefiniton}{(Probabilistic Safety).}
Consider the training dataset $\mathcal{D}$. Let $T\subseteq \mathbb{R}^m$ and fix $x^* \in T.$  For $\delta> 0$
call 
\begin{align*}
    \phi_1^{i} (x^*,&T,\delta\, | \, \mathcal{D})=\\
     &P( \exists x' \in T\, s.t.\, \big( \GP^{(i)}(x^*)-\GP^{(i)}(x') \big ) > \delta \, | \, \mathcal{D}),
\end{align*}
   where $\GP^{(i)}$ is the i-th component of $\GP.$ 
Then we say that component $i$ in $\GP$ is safe with probability $1-\epsilon >0$ for $x^*$ with respect to set $T$ and perturbation $\delta> 0$ iff
\begin{align}
   \phi_1^{i} (x^*&,T,\delta |  \mathcal{D}) \leq \epsilon .
   \label{def:SAFETYGPSingle}
\end{align}
\end{myprobl}
\noindent
Intuitively,  we consider a test point $x^*$ and a compact set $T$ containing $x^*$, and compute the probability that the predictions of $\GP$ remain $\delta-$close for each $x' \in T$.
We consider the components of the GP individually and with sign, enabling one-sided analysis. 
Note that $T$ is composed of an uncountable number of points, making the probability computation challenging.
{Moreover, Problem \ref{BoundedVariationSingleCOmponentDefiniton} will still represent a sound notion of safety even in the case that a distribution on the input space can be assumed.}
Problem \ref{BoundedVariationSingleCOmponentDefiniton} can be generalized to  local invariance of $\GP$ with respect to a given metric (Problem \ref{BoundedVariationDefinition} below). In the next section, for the corresponding solution, 
we will work with the $L_1$ norm, but all the results can be easily extended to any $L_p$ norm, including $L_{\infty}$. 

\begin{myprobl}{(Probabilistic Invariance)}\label{BoundedVariationDefinition}
Consider the training dataset $\mathcal{D}$. 
Let $T\subseteq \mathbb{R}^m$ and assume $x^* \in T.$
For metric $|| \cdot ||_d: \mathbb{R}^n \to \mathbb{R}_{\geq 0}$ and $\delta >0$ call
\begin{align*}
    \phi_2(x^*,&T,\delta | \mathcal{D})=P( \exists x' \in T\, s.t.\,  || \GP(x')-\GP(x^*)) ||_d > \delta  | \mathcal{D} )
\end{align*}
Then we say that $\GP$  is $\delta-$invariant with respect to metric $|| \cdot ||_d$  for $x^*$ in $T$ and perturbation $\delta> 0$ with probability $1-\epsilon>0$ iff
\begin{align}
    \label{def:SAFETYGPextended}
    \phi_2(x^*,&T,\delta | \mathcal{D}) \leq  \epsilon.
\end{align}
\end{myprobl}
Probabilistic invariance, as defined in Problem \ref{BoundedVariationDefinition}, bounds the probability that each function sampled from $\GP$ remains within a distance of at most $\delta$ to the initial point. 
Note that both Problem \ref{BoundedVariationSingleCOmponentDefiniton} and  \ref{BoundedVariationDefinition} quantify the probability of how the output of a learning process changes its value in a set around a given test input point, which implicitly gives  probabilistic guarantees for local robustness against adversarial examples.
In the next section, in Theorem \ref{BoundsSIngleBoundedVariation} and \ref{Theorem-Invariance}, we give analytic upper bounds for Problem \ref{BoundedVariationSingleCOmponentDefiniton} and \ref{BoundedVariationDefinition}.
In fact, analytic distributions of the supremum of a GP, which would allow one to solve the above problems, are known only for a very limited class of GPs (and always for GPs evolving over time) \cite{adler2009random}, making exact computation impossible. 
However, first, we illustrate the intuition behind the problems studied here on a GP regression problem. 

\begin{myexam}\label{ex:run_ex1}
We consider a regression problem taken from  \cite{bach2009exploring},
where we generate 128 samples from a random two-dimensional covariance matrix, and define labels as a (noisy) quadratic polynomial of the two input variables.
We train a GP with squared-exponential kernel on this dataset, using a maximum likelihood estimation of the kernel hyper-parameters \cite{rasmussen2004gaussian}. 
The mean and variance of the GP obtained after training are plotted in Figure \ref{fig:example_mean_var}, along with the samples used for training.
\begin{figure}
\centering
\subfloat[Mean.]{\includegraphics[width = .49\columnwidth]{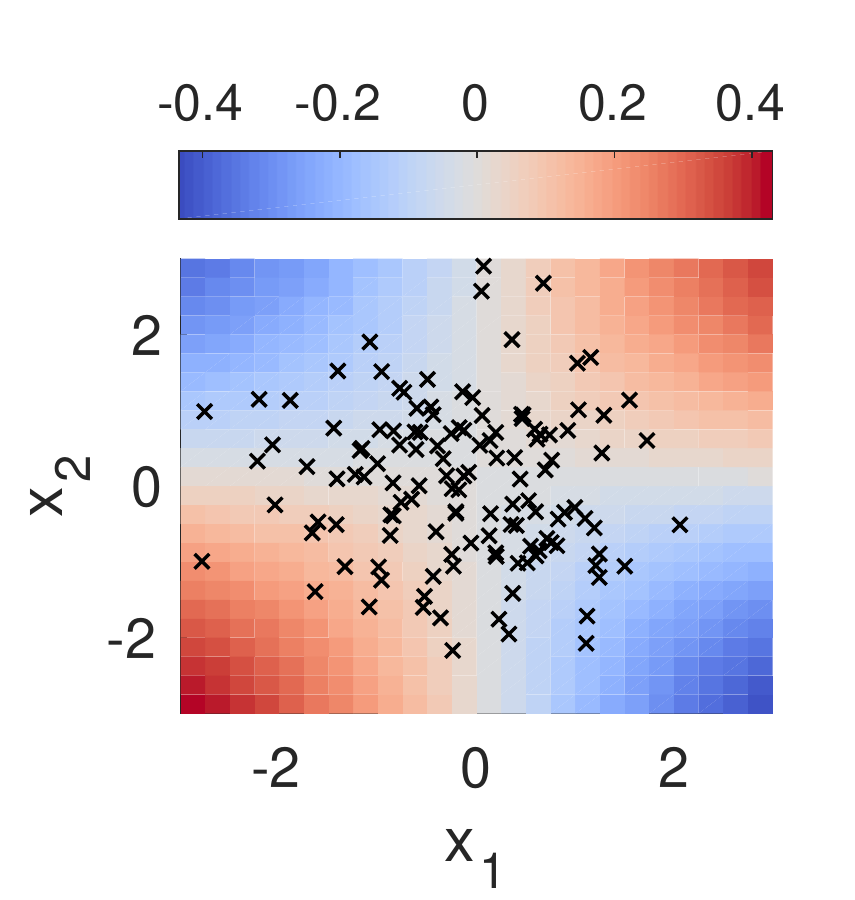}} \hfill
\subfloat[Variance.]{\includegraphics[ width = .49\columnwidth]{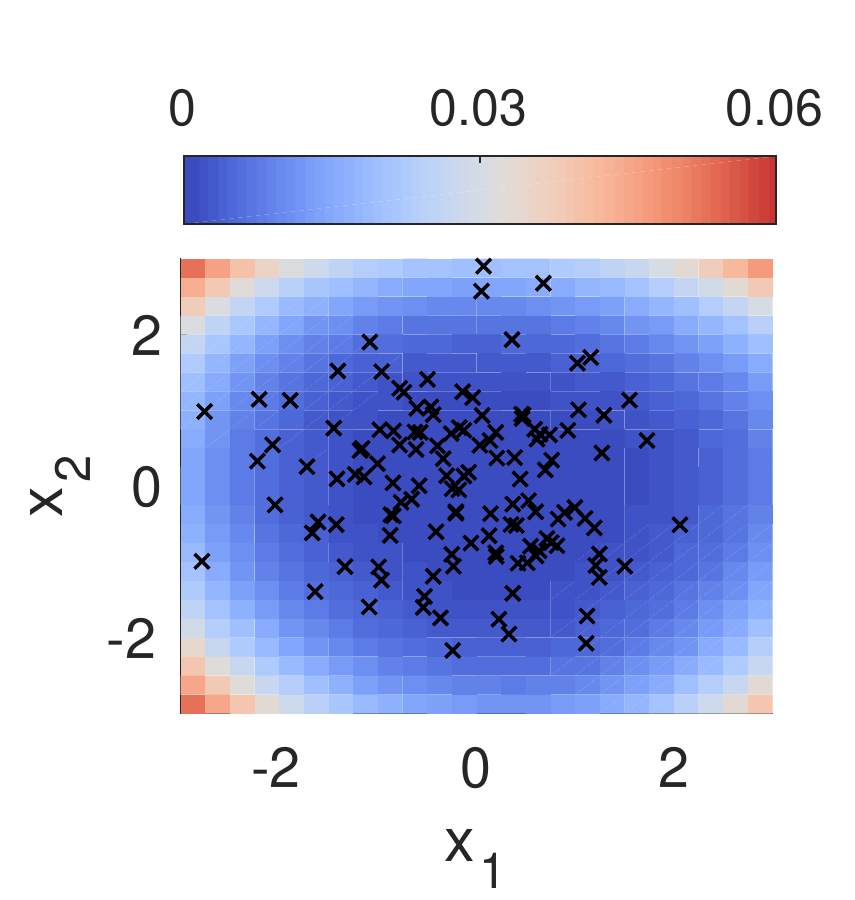}}
\caption{Results of GP training.}\label{fig:example_mean_var}
\end{figure}
Consider the origin point $x^o = \left( 0 , 0 \right)$, let $\gamma = (0.1,0.1)$  and define $T^o_\gamma = [x^o - \gamma, x^o + \gamma ]$. As $x^o$ is a saddle point for the mean function, variations of the mean around it are relatively small.
Analogously, the variance function exhibits a flat behaviour around $x^o$, meaning greater confidence of the GP in performing predictions around $x^o$. 
As such we expect realizations of the GP to be consistently stable in a neighbourhood of $x^o$, which in turn translates to low values for  $\phi_1 (x^o,T^o_\gamma,\delta )$ and $\phi_2 (x^o,T^o_\gamma,\delta),$ where in $\phi_1$ and $\phi_2$, to simplify the notation, we omit the dataset used for training.
On the other hand, around $x^* = (3,3)$ the a-posteriori mean changes quickly and the variance is high, reflecting higher uncertainty. 
Hence, letting $T^*_\gamma = [x^* - \gamma, x^* + \gamma ]$, we expect the values of $\phi_1 (x^*,T^*_\gamma,\delta)$ and $\phi_2 (x^*,T^*_\gamma,\delta)$
to be greater than those computed for $x^o$.

In the next section we show how $\phi_1 (x,T_\gamma,\delta)$ and $\phi_2 (x,T_\gamma,\delta)$ can be computed to quantify the uncertainty and variability of the predictions around $x^o$ and $x^*.$


\end{myexam}

\section{Theoretical Results}
Since $\GP$ is a Gaussian process, its distribution is completely defined by its mean $\mu: \mathbb{R}^m \to \mathbb{R}^n$ and covariance (or kernel) function $\Sigma: \mathbb{R}^m \times \mathbb{R}^m \to \mathbb{R}^{n\times n}.$  
Consider a set of training data $\mathcal{D}=\{\GP(x_i)=y_i, i \in \{1,...,\tss \} \},$ and call $\mathbf{y}=[y_1,...,y_{\tss}].$ 
Training $\GP$ in a Bayesian framework is equivalent to computing the distribution of ${\GP}$ given the dataset $\mathcal{D}$, that is, the distribution of the process
$$ \bar{\GP}=\GP \, |\,\mathcal{D}.  $$
Given a test point $x^*\in \mathbb{R}^m$ and $x_1,...,x_{\tss}$ training inputs in $\mathcal{D}$, consider the joint distribution
$[\GP(x^*),\GP(x_1),...,\GP(x_{\tss})]$, which is still Gaussian  with mean and covariance matrix given by
$$ \mu=\begin{bmatrix}
       \mu(x^*) ,
       \mu(x_1) ,...,   \mu(x_{\tss})
     \end{bmatrix}
     \quad    \Sigma=\begin{bmatrix}
       \Sigma_{x^*,x^*} & \Sigma_{x^*,\mathcal{D}}           \\[0.3em]
       \Sigma_{x^*,\mathcal{D}}^T & \Sigma_{\mathcal{D},\mathcal{D}}
     \end{bmatrix}, $$ 
     where $\Sigma_{D,D}$ is the covariance matrix relative to vector $[\GP(x_1),...,\GP(x_{|\mathcal{D}|})].$
     Then, it follows that $\bar{\GP}$ is still Gaussian with mean and covariance matrix defined as follows:
\begin{align}
   \label{Eq:COnditionMean}& \bar{\mu}(x^*)= \mu(x^*) + \Sigma_{x^*,\mathcal{D}}\Sigma_{\mathcal{D},\mathcal{D}}^{-1}(\mathbf{y}-\mu_{\mathcal{D}})\\
   &\bar{\Sigma}_{x^*,x^*}= \Sigma_{x^*,x^*}-\Sigma_{x^*,\mathcal{D}}\Sigma_{\mathcal{D},\mathcal{D}}^{-1}\Sigma_{x^*,\mathcal{D}}^T,\label{Eq:ConditionalVariance}
\end{align} 
 where $\mu_{\mathcal{D}}=[\mu(x_1),...,\mu(x_{\tss})].$
Hence, for GPs the distribution of $\bar{\GP}(x^*)$ can be computed exactly. 

Given two test points $x^*_1,{x^*_2}$ and $x^*=[x^*_1,{x^*_2}]$, the above calculations can still be applied to compute the joint distribution
$$\bar{\GP}(x^*)=\big([\GP(x^*_1),\GP(x^*_2)]\, |\,\mathcal{D}\big) . $$ 
In particular, $\bar{\GP}(x^*)$ is still Gaussian and with mean $\bar{\mu}$ and covariance matrix $\bar{\Sigma}$ given by Eqns \eqref{Eq:COnditionMean} and \eqref{Eq:ConditionalVariance} 
but with $\mu(x^*)=[\mu(x^*_1),\mu(x^*_2)]$ and
$    \Sigma_{x^*,x^*}=\begin{bmatrix}
       \Sigma_{x^*_1,x^*_1} & \Sigma_{x^*_1,x^*_2}           \\[0.3em]
      \Sigma_{x^*_1,x^*_2}^T & \Sigma_{x^*_2,x^*_2}
     \end{bmatrix}. $
     From $\bar{\GP}(x^*)$ we can obtain  the distribution of the following random variable
    $$\GP^o(x^*_1,x^*_2)=\big( \GP(x^*_1)-\GP(x^*_2)\big)\,  | \, \mathcal{D},$$
which represents the difference of $\GP$ at two distinct test points after training. It is straightforward to show that,  given $B\in \mathbb{R}^{n\times 2 n}$ such that $B=[I;-I],$ where $I$ is the identity matrix of dimension $n,$
     $\GP^o(x^*_1,x^*_2)$ is Gaussian with mean and variance
     $$\mu^{o}(x^*_1,x^*_2)=B \bar{\mu}(x^*) \quad \quad    \Sigma^o_{x^*_1,x^*_2}=B \bar{\Sigma}_{x^*,x^*} B^T. $$
$\GP^o(x^*_1,x^*_2)$ is the distribution of how $\GP,$ after training, changes with respect to two different test  points.
However, to solve Problem \ref{BoundedVariationSingleCOmponentDefiniton} and \ref{BoundedVariationDefinition}, we need to take into account all the test points in $T\subseteq \mathbb{R}^m$ and compute the probability that in at least one of them $\GP^o$ exits from a given set of the output space. 
This is done in Theorem \ref{BoundsSIngleBoundedVariation} by making use of the Borell-TIS inequality and of the Dudley's entropy integral \cite{adler2009random,dudley1967sizes}.
The above inequalities allow one to study Gaussian processes by appropriately defining a metric on the variance of the GPs.
In order to define such a metric we call $\hat \GP^o$ the GP with the same covariance matrix as $\GP^o$ but with zero mean and $\hat \GP^{o,(i)}$ its $i$-th component.
For $i \in \{1,...,n\}$, a test point $ x^* \in \mathbb{R}^{m}, $ and $x_1,x_2 \in \mathbb{R}^m$  we define the (pseudo-)metric $d^{(i)}_{x^*}(x_1,x_2)$ by
\begin{align}\label{Eqn:dNormdefinition}
   d^{(i)}_{x^*}(x_1,x_2)=&\sqrt{ \mathbb{E}[(\hat \GP^{o,(i)}(x^*,x_1)- \hat \GP^{o,(i)}(x^*,x_2))^2]  }\\
   =&\sqrt{\mathbb{E}[(\hat \GP^{(i)}(x_2)- \hat \GP^{(i)}(x_1))^2]}, \nonumber  
\end{align} 
where $\hat \GP^{(i)}$ is the $i$-th component of the zero-mean version of $\bar{\GP}$. Note that $d^{(i)}_{x^*}(x_1,x_2)$ does not depend on $x^*.$
Additionally, we assume there exists a constant $K_{x^*}^{(i)}>0$ such that for a compact $T\subseteq \mathbb{R}^m$ and $x^*,x_1,x_2 \in T$\footnote{Note that here we work with the $L_2$ norm, but any other $L_p$ metric would work.} 
$$ d^{(i)}_{x^*}(x_1,x_2)\leq K^{(i)}_{x^*} \vert \vert x_1-x_2 \vert \vert_2. $$
Now, we are finally ready to state the following theorem.
\begin{theorem}\label{BoundsSIngleBoundedVariation}
Assume $T\subseteq \mathbb{R}^m,m>0,$ is a hyper-cube with layers of length $D>0$.
 For $ {x^*} \in T, \delta >0,$ and $i \in \{1,\ldots,n \}$ let
    \begin{align*}
        &\eta_i = \delta -\big( \, sup_{x \in T}\mu^{o,(i)}({x^*}, x)\,   + \\
        &12 \int_{0}^{\frac{1}{2}sup_{x_1,x_2 \in T} d^{(i)}_{{x^*}}(x_1,x_2)}\sqrt{ ln \left(\big( \frac{\sqrt{m} K^{(i)}_{{x^*}} D\, }{ z}+ 1\big)^m  \right) } d z\big).
    \end{align*}
   Assume $\eta_i > 0$. Then, it holds that
    \begin{equation*}
        \phi_1^{i} (x^*,T,\delta |  \mathcal{D}) \leq \hat{\phi}_1^{i} (x^*,T,\delta |  \mathcal{D}) :=  e^{ -\frac{\eta_i^2}{2 \xi^{(i)} }},
    \end{equation*}
    where $\xi^{(i)} = \sup_{x\in  T} \Sigma^{o,(i,i)}_{{x^*},x}$ is the supremum of the component $(i,i)$ of the covariance matrix $\Sigma^{o}_{{x^*},x}$.
\end{theorem}
\begin{proof}{[Sketch.]}
{
     \begin{align*}
     & \phi_1^{i} (x^*,T,\delta |  \mathcal{D})\\
     &\quad \quad \text{(By definition of $\phi_1$)}\\
         =& P( \exists x \in T\, s.t.\, \big( \GP^{(i)}(x)-\GP^{(i)}(x^*)>\delta \, | \, \mathcal{D} \big)\\ & \quad \quad \text{(By definition of supremum)}\\
         =&P\big( \sup_{x \in T}\,  \GP^{o,(i)}(x^*, x)>\delta \big)\\ & \quad \quad \text{(By linearity of GPs)}\\
         =&P\big( \sup_{x \in T}\,\hat \GP^{o,(i)}(x^*, x) + \mathbb{E}[\GP^{o,(i)}(x^*, x)]>\delta \big)\\ & \quad \quad \text{(By definition of supremum)}\\
         \leq &P\big( \sup_{x \in T}\, \hat \GP^{o,(i)}(x^*, x) >\delta- sup_{x_1 \in T}\mathbb{E}[\GP^{o,(i)}( x^*,x_1)] \big).
     \end{align*}
     where $\hat \GP^{o,(i)}(x^*, x)$ is the zero mean Gaussian process with same variance of $\GP^{o,(i)}(x^*, x).$ 
The last inequality can be bounded from above using the Borell-TIS inequality \cite{adler2009random}.
 To use such an inequality we need to derive an upper bound of $\mathbb{E}[sup_{t \in T} \GP^{o,(i)}(x^*,x)]$. This can be done by employing the Dudley's Entropy integral \cite{dudley1967sizes}. 
 
The extended version of the proof can be found in the Appendix.
}
\end{proof}
\noindent
 In Theorem \ref{BoundsSIngleBoundedVariation} {we derive $\hat{\phi}_1^{i} (x^*,T,\delta |  \mathcal{D})$ as an} upper bound for $\phi_1^{i} (x^*,T,\delta |  \mathcal{D})$. Considering that $\GP$ is a Gaussian process, it is interesting to note that the resulting bounds still follow an exponential distribution.
From Theorem \ref{BoundsSIngleBoundedVariation} we have the following result.
\begin{theorem}\label{Theorem-Invariance}
Assume $T\subseteq \mathbb{R}^m,m>0$ is a hyper-cube with layers of length $D>0$.
 For $ {x^*} \in T, \delta >0$ let
    \begin{align*}
        &\bar \eta_i = \frac{\delta\, -\, sup_{x \in T}|\mu^{o}({x^*}, x)|_1}{n}\,   - \\
        & 12 \int_{0}^{\frac{1}{2}sup_{x_1,x_2 \in T} d^{(i)}_{{x^*}}(x_1,x_2)} \sqrt{ln \left(\big( \frac{\sqrt{m} K^{(i)}_{{x^*}} D\, }{ z}+ 1\big)^m  \right)}d z.
    \end{align*}
   For each $i \in \{1,...,n\}$ assume $\bar \eta_i > 0$. Then, it holds that
    \begin{equation*}
        \phi_2 (x^*,T,\delta |  \mathcal{D}) \leq \hat{\phi}_2 (x^*,T,\delta |  \mathcal{D}) := 2 \sum_{i=1}^{n} e^{ -\frac{\bar \eta_i^2}{2 \xi^{(i)} }},
    \end{equation*}
    where $\xi^{(i)} = \sup_{x\in  T} \Sigma^{o,(i,i)}_{{x^*},x}$.
\end{theorem}
\noindent
Note that in Theorem \ref{BoundsSIngleBoundedVariation} and \ref{Theorem-Invariance} we assume that $T$ is a hyper-cube. However, proofs of both theorems (reported in the Supplementary Materials) can be easily extended to more general compact sets, at a cost of more complex analytic expressions or less tight bounds.

Both theorems require the computation of a set of constants, which depends on the particular kernel. 
In particular, $\xi^{(i)}$ and $ \sup_{x \in T}\mu^{o,(i)}(x^*,x)$ are upper bound of variance and mean a-posteriori while, for a test point $x^*,$ $K^{(i)}_{x^*} $ and $\sup_{x_1,x_2 \in T} d^{(i)}_{x^*}(x_1,x_2)$ represent local Lipschitz constant and upper bound for $d^{(i)}_{x^*}$ in $T.$ 
In the next section, we show how these constants can be computed. 

\begin{myexam}
   \begin{figure*}[h]
    \centering
    {\includegraphics[width = 0.98 \columnwidth]{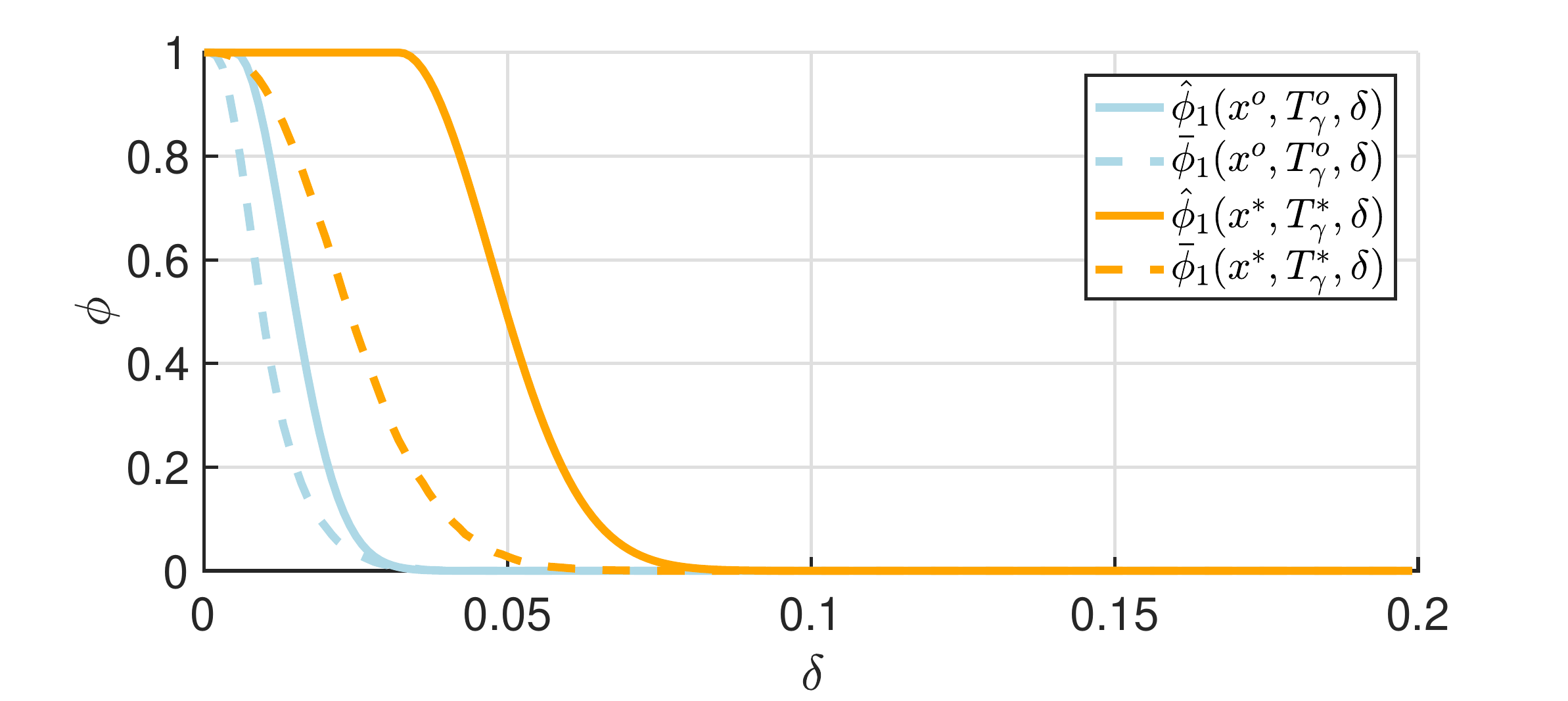}} 
    {\includegraphics[width = 0.98 \columnwidth]{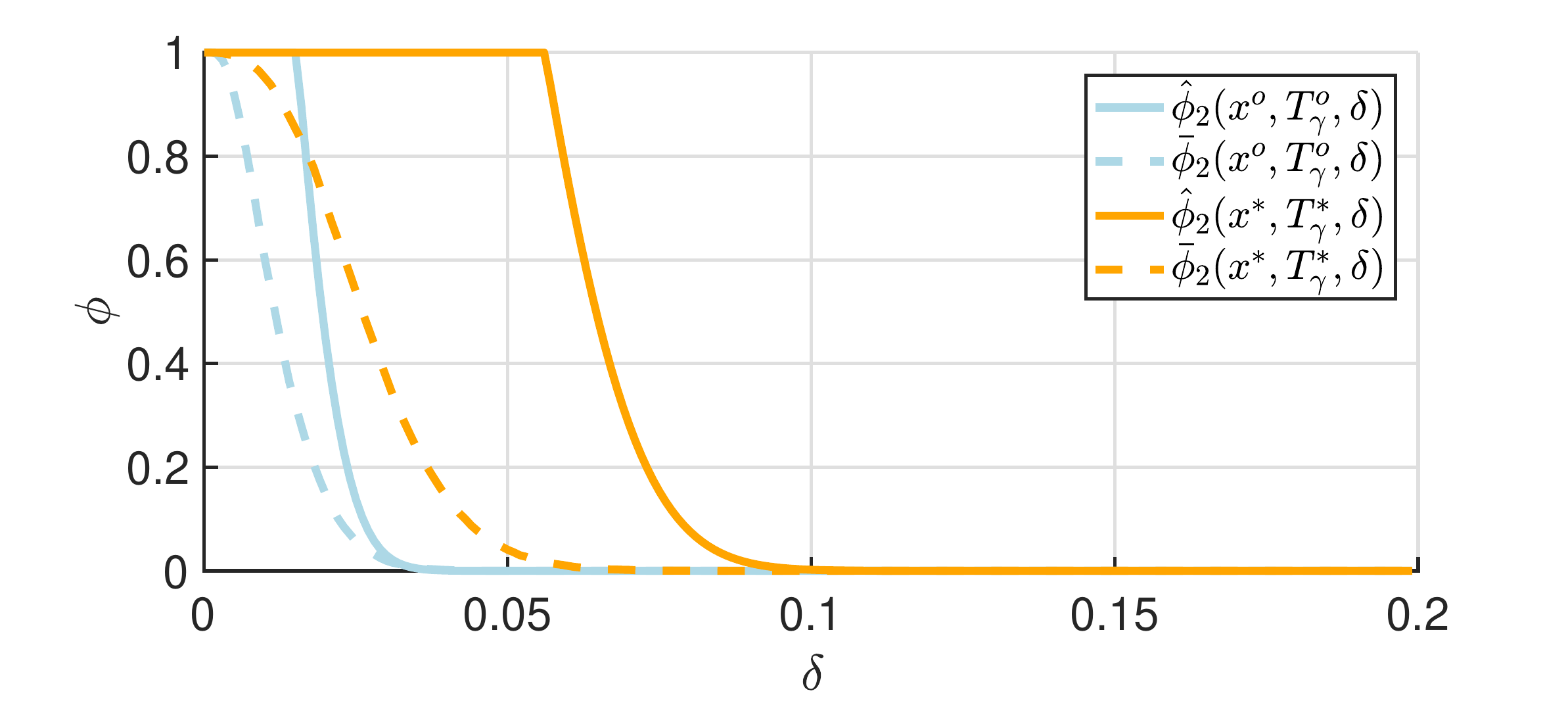}}
    \caption{{Upper bounds (solid lines) and sampling approximation (dashed lines) for $\phi_1$ (left plot) and $\phi_2$ (right plot) on $x^o$ and $x^*$.}}
    \label{fig:phi_all_run_ex}
\end{figure*}

We illustrate the upper bounds for $\phi_1$ and $\phi_2$, as given by Theorem \ref{BoundsSIngleBoundedVariation}  and \ref{Theorem-Invariance}, 
on the GP introduced in Example \ref{ex:run_ex1}. 
Figure \ref{fig:phi_all_run_ex} shows the values obtained for $\hat{\phi}_1$ and $\hat{\phi}_2$ on $x^o$ and $x^*$ for  $\delta$ between $0$ and $0.2$.
We observe that values computed for $x^*$ are consistently greater than those computed for $x^o$, which captures and probabilistically quantifies the increased uncertainty of the GP around $x^*$, as well as the increased ratio of mean variation around it (see Figure \ref{fig:example_mean_var}).
Notice also that values for $\hat{\phi}_1$ are always smaller than the corresponding $\hat{\phi}_2$ values. This is a direct consequence of the fact that probabilistic invariance is a stronger requirement than probabilistic safety, as defined in Problem \ref{BoundedVariationSingleCOmponentDefiniton}, as the latter is not affected by variations that tend to increase the value of the GP output (translating to increased confidence in classification settings).
{In Figure \ref{fig:phi_all_run_ex} we also compare the upper bounds obtained with estimation for $\phi_1$ and $\phi_2$ based on sampling of the GP in a discrete grid around the test points.
Specifically, we sample 10000 functions from the GP and evaluate them in 2025 points uniformly spaced around the test point.
We remark that this provides us with just an empirical under-approximation of the actual values of $\phi_1$ and $\phi_2$, referred to as $\bar{\phi}_1$ and $\bar{\phi}_2$ respectively.
Note also that $\bar{\phi}_1$ and $\bar{\phi}_2$ have an exponential decay. 
The results suggest that the approximation is tighter around $x^o$ than around $x^*$.
In fact, higher variance will generally imply a looser bound, also due to the over-approximations introduced in the computation of the constants required in the theorems.
}
\end{myexam}

\subsection{Constant Computation}  
We introduce a general framework for the computation of the constants involved in the bounds presented in the previous section with an approach based on a generalisation of that of \cite{jones1998efficient} for squared-exponential kernels in the setting of Kriging regression. %
%
%
Namely, we assume the existence of a suitable decomposition of the kernel function as $\Sigma_{x,x_i} = \psi_\Sigma \left( \varphi_\Sigma \left( x,x_i \right)  \right)$ for all $x$ and $x_i \in \mathbb{R}^m$, such that:
\begin{enumerate}
    \item $\varphi_\Sigma  : \mathbb{R}^m \times \mathbb{R}^m \rightarrow \mathbb{R} $ is a continuous function;
    \item $\psi_\Sigma : \mathbb{R} \rightarrow \mathbb{R}$ is differentiable, with $\frac{d \psi_\Sigma }{d  \varphi_\Sigma}$ continuous;
    \item $\sup_{x \in T} \sum_{i = 1}^{\tss} c_i \varphi_\Sigma \left(x, x_i \right) $ can be computed for each $c_i \in \mathbb{R}$ and $x_i \in \mathbb{R}^m$, $i = 1, \ldots, \tss$.
\end{enumerate}
While assumptions 1 and 2 usually follow from smoothness of the kernel used, assumption 3 depends on the particular $\varphi_\Sigma$ defined.
Intuitively, $\varphi_\Sigma$ should represent the smallest building block of the kernel which captures the dependence on the two input points.
For example for the squared exponential kernel this has the form of a separable quadratic polynomial so that assumption 3 is verified.
Similarly, for the ReLU kernel $\varphi_\Sigma$ can be defined as the dot product between the two input points.
{A list of commonly used kernels that satisfy assumptions 1 to 3 is given in the Supplementary Materials, along with valid decomposition functions $\varphi_\Sigma $ and $\psi_\Sigma $.}

Assumptions 1 and 2 guarantee the existence for every $x_i \in \mathbb{R}^m$ of a set of constants  $a^i_L$, $b^i_L$, $a^i_U$ and $b^i_U$ such that:
\begin{equation}\label{eq:lin_approx}
    a_L^i + b_L^i \varphi_\Sigma \left( x, x_i \right) \leq \Sigma_{x,x_i} \leq a_U^i + b_U^i \varphi_\Sigma \left( x, x_i \right) \quad \forall x \in T.
\end{equation}
In fact, 
it follows from those that $\psi_\Sigma$ has a finite number of flex points.
Hence, we can iteratively find lower and upper approximation in convex and concave parts, and merge them together as detailed in the Supplementary Material. 
The key point is that, due to linearity, this upper and lower bound on the kernel can be propagated through the inference equations for Gaussian processes, so as to obtain lower and upper linear bounds on the a-posteriori mean and variance with respect to $\varphi_\Sigma$.
Thanks to assumption 3, these bounds can be solved for optimal points exactly, thus providing formal lower and upper values on optimization over a-posteriori mean and variance  of the Gaussian process in $T$.
The described approach can be used to compute $\xi^{(i)}$, $K^{(i)}_{x^*}$ and $\sup_{x \in T}\mu^{o}(x^*,x)$ . In the following subsection we give details for the  computation of $\sup_{x \in T}\mu^{o}$.
We refer to the Supplementary Materials for the details for the other constants and for squared-exponential and ReLU kernels.

\subsubsection{Mean Computation}
As $x^*$ is fixed, we have that  $\sup_{x \in T}\mu^{o} = \bar{\mu}(x^*) - \inf_{x \in T} \bar{\mu}(x)$, hence we need just to compute $\inf_{x \in T} \bar{\mu}(x)$.
Using Eqn \eqref{eq:lin_approx} we can compute a lower and upper bound to this inferior, which can be refined using standard branch and bound techniques.
Let $t =\Sigma_{\mathcal{D},\mathcal{D}}^{-1}(\mathbf{y}-\mu_{\mathcal{D}})$, then by the inference formula for Gaussian processes and Eqn \eqref{eq:lin_approx} we have that:
\begin{align*}
    \bar{\mu}(x) = \sum_{i = 1}^\tss  t_i \Sigma_{x,x_i} \geq \sum_{i = 1}^\tss t_i \left( a^i + b^i \varphi_\Sigma \left( x, x_i \right) \right) \quad \forall x \in T
\end{align*}
where we choose: 
$
  (a^i , b^i) = \begin{cases}
  (a^i_L , b^i_L), & \text{if $t_i \geq 0$}.\\
  (a^i_U , b^i_U), & \text{otherwise}.
  \end{cases}.
$
Let $\bar{x}$ be an inferior point for $ x \in T$ to the right-hand side Equation (that can be computed thanks to Assumption 3, with $c_1 := t_i b^i $), then, by the definition of inferior we have that:
\begin{align*}
    \bar{\mu}(\bar{x})  \geq \inf_{x \in T} \bar{\mu}(x) \geq \sum_{i = 1}^\tss t_i a_i + \sum_{i = 1}^\tss t_i  b_i \varphi_\Sigma \left( \bar{x}, x_i \right).
\end{align*}
The latter provide bounds on $\inf_{x \in T} \bar{\mu}(x)$ that can be used within a branch and bound algorithm for further refinement. 

\subsection{Computational Complexity}
 Performing inference with GPs has a cost that is $\mathcal{O}(\tss^3),$ where $\tss$ is the size of the dataset. Once inference has been performed the cost of computing upper and lower bounds for $\sup_{x \in T} \mu^{o,(i)}(x^*,x)$ is $\mathcal{O}(\tss C),$ where $C$ is a constant that depends on the particular kernel. For instance, for the squared-exponential kernel $C=1,$ while for the ReLU kernel (Eqn \eqref{KernelSquaredExponential}) $C=L,$ where $L$ is the number of layers of the corresponding neural network. The computation of the bounds for $\xi^{(i)}$ requires  solving a convex {quadratic} problem in $m+\tss + 1$ variables, while $K^{(i)}_{x^*}$ and $\sup_{x_1,x_2 \in T}d_{x^*}^{(i)}(x_1,x_2)$ can be bounded in constant time. Refining the bounds with a branch and bound approach has a worst-case cost that is exponential in $m,$ the dimension of the input space.
Hence,  sparse approximations, which mitigate the cost of performing inference with GP \cite{seeger2003fast}, are appealing.



\section{Experimental Evaluation: Robustness Analysis of Deep Neural Networks}

\begin{figure*}[!h]
\centering
\includegraphics[width = .5\textwidth, clip = on, trim = 0mm 50mm 0mm 0mm]{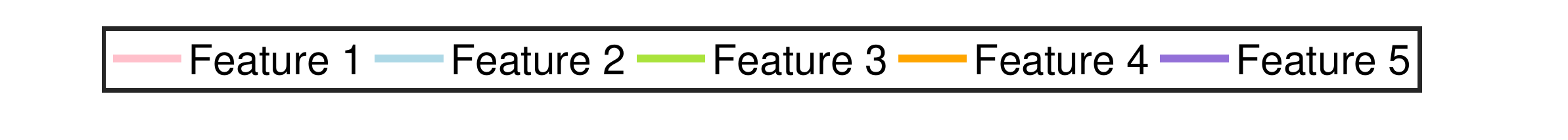}\\
\hspace*{2cm}{\includegraphics[ width = .15\textwidth]{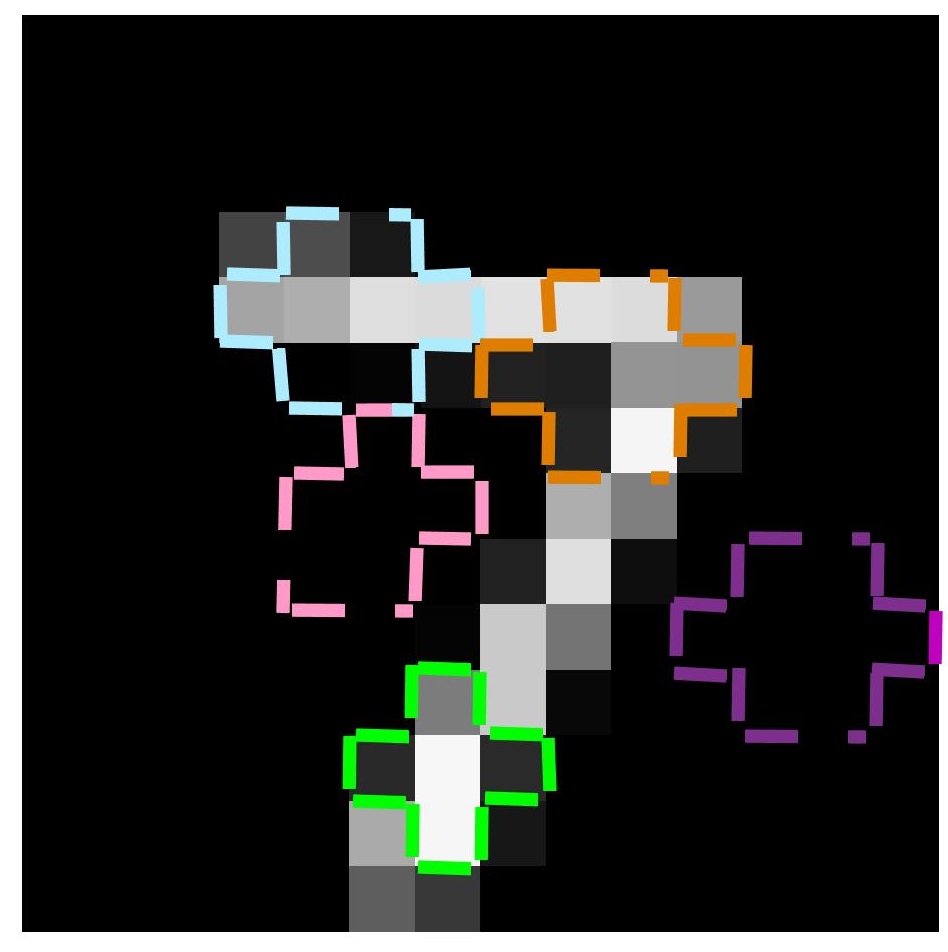}} \hfill
{\includegraphics[width = .15\textwidth]{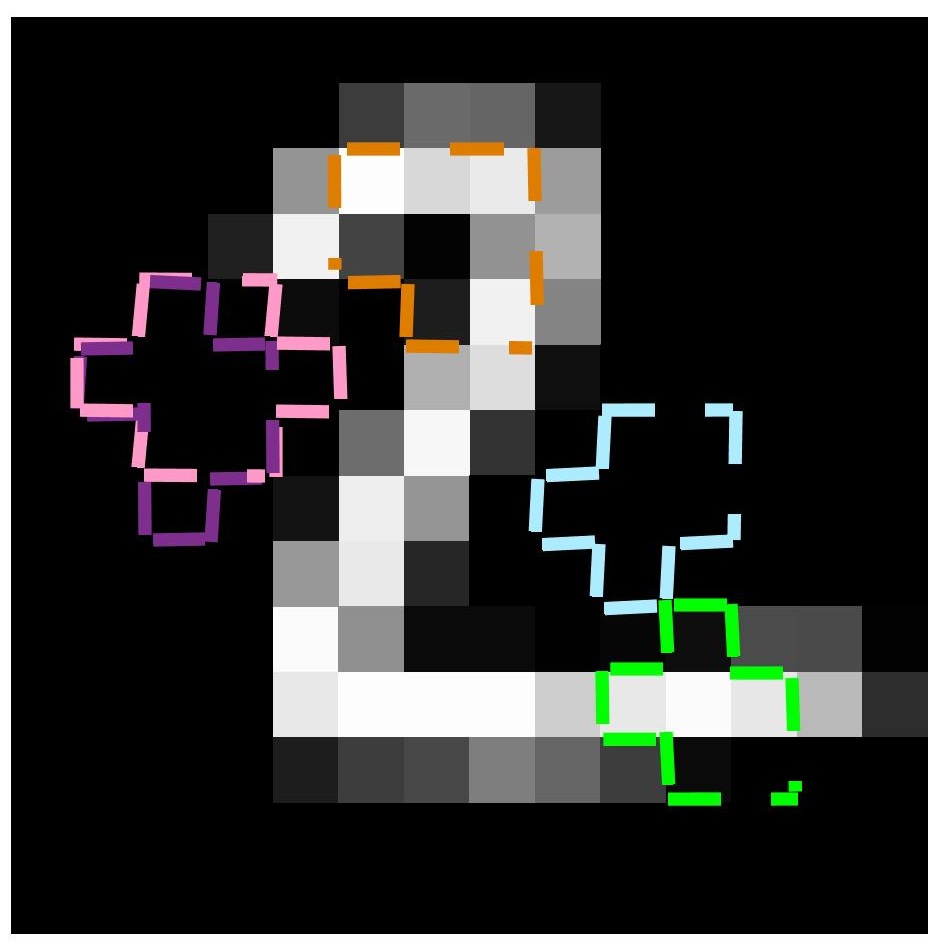}} \hfill
{\includegraphics[width = .15\textwidth]{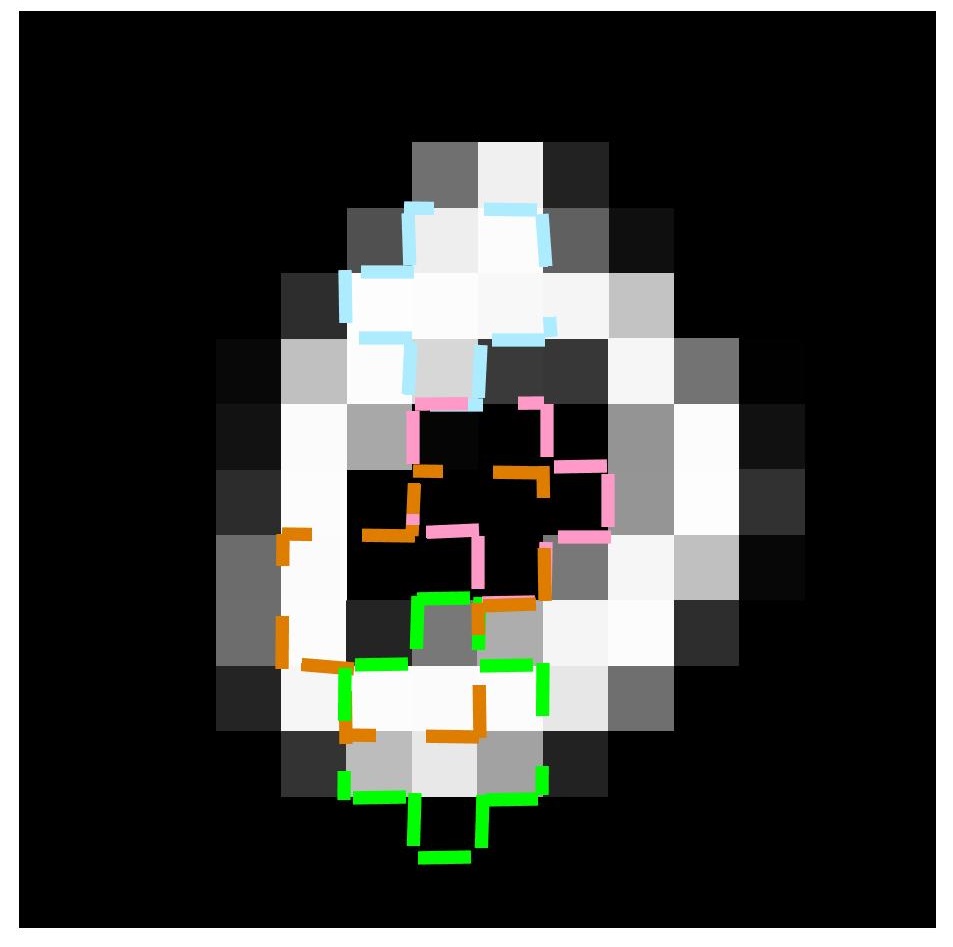}}\hspace*{1.3cm} \\
{\includegraphics[clip = on,trim = 0mm 0mm 15mm 0mm, width = .33\textwidth]{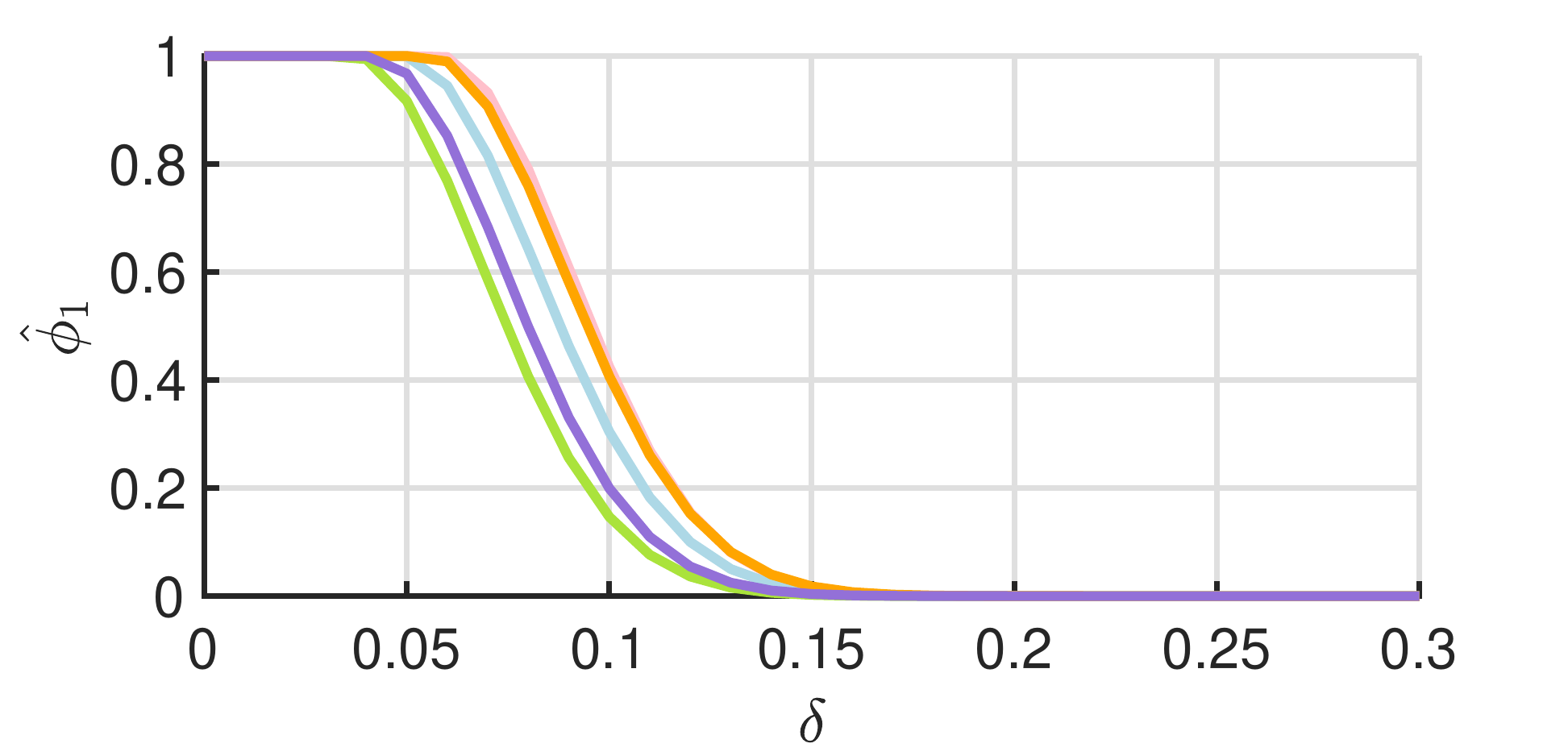}} \hfill
{\includegraphics[clip = on,trim = 0mm 0mm 15mm 0mm,width = .33\textwidth]{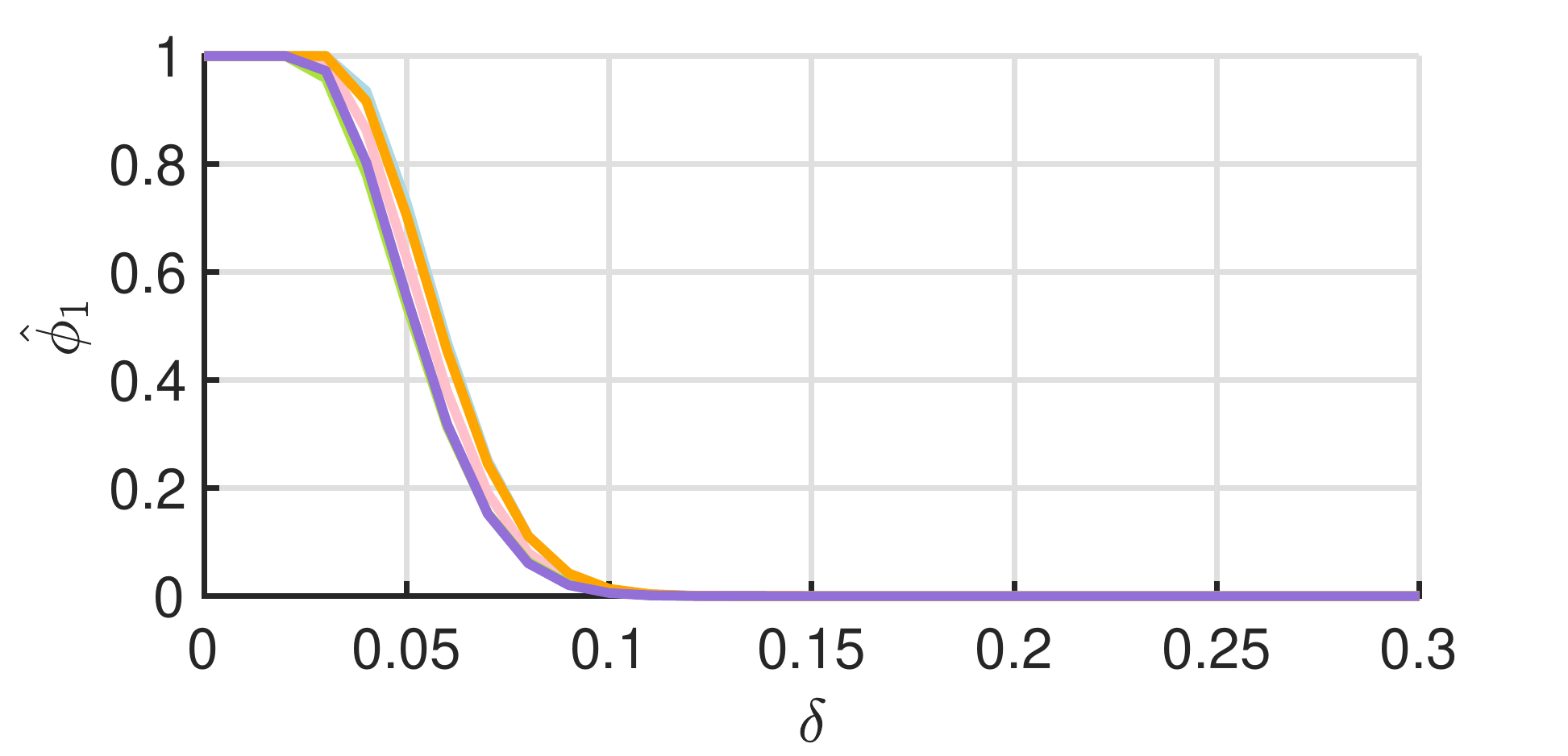}} \hfill
{\includegraphics[clip = on,trim = 0mm 0mm 15mm 0mm,width = .33\textwidth]{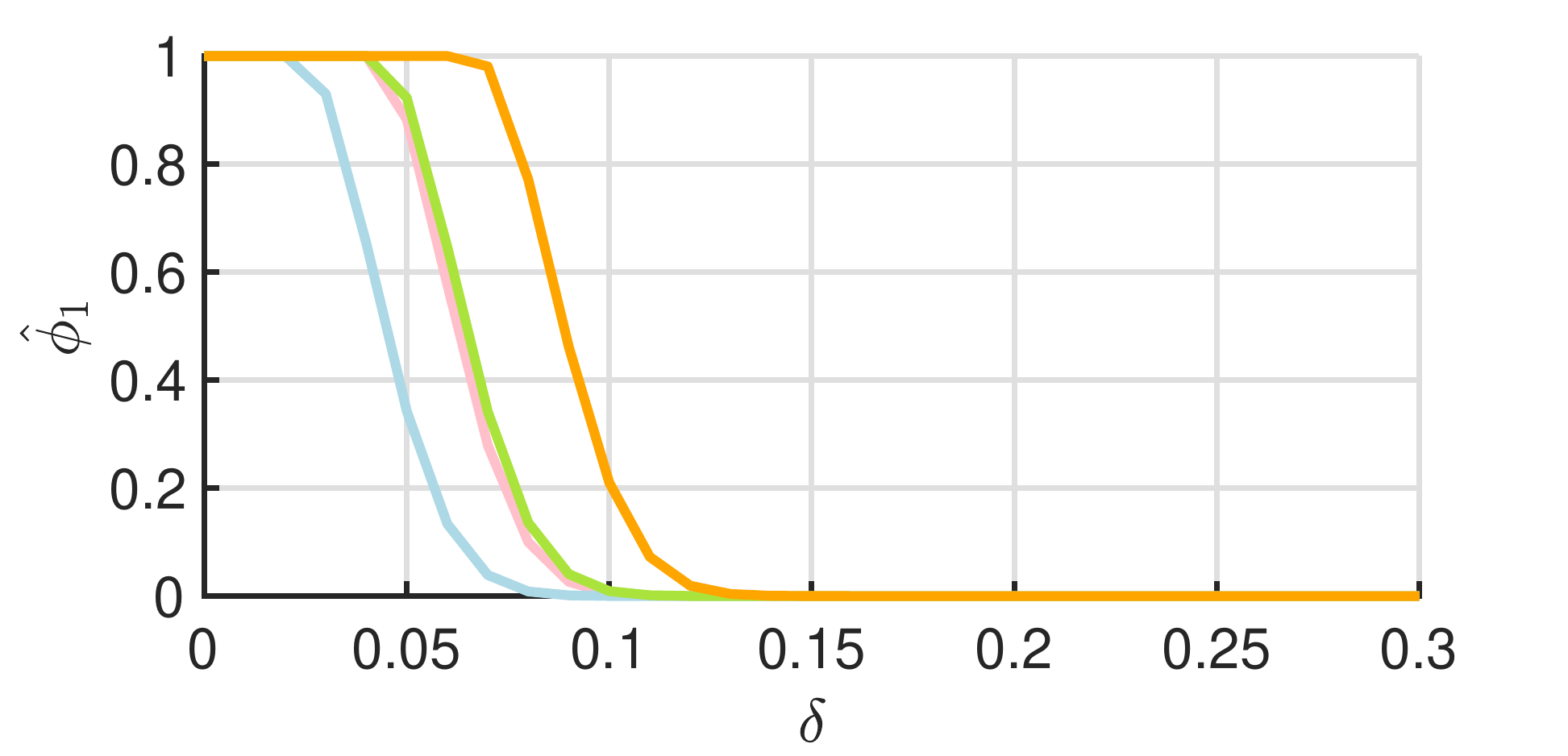}} \\ 
{\includegraphics[clip = on,trim = 0mm 0mm 15mm 0mm,width = .33\textwidth]{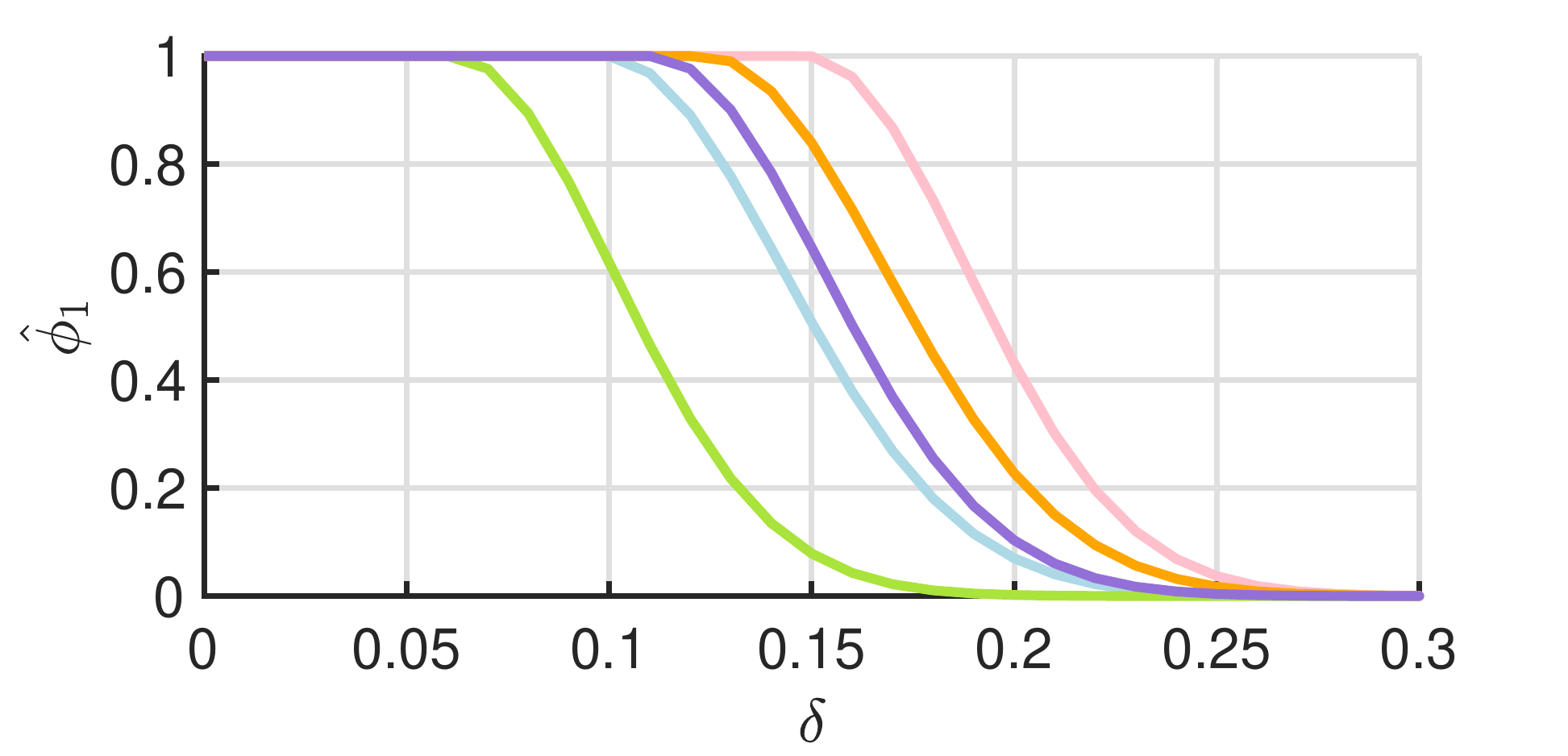}} \hfill
{\includegraphics[clip = on,trim = 0mm 0mm 15mm 0mm,width = .33\textwidth]{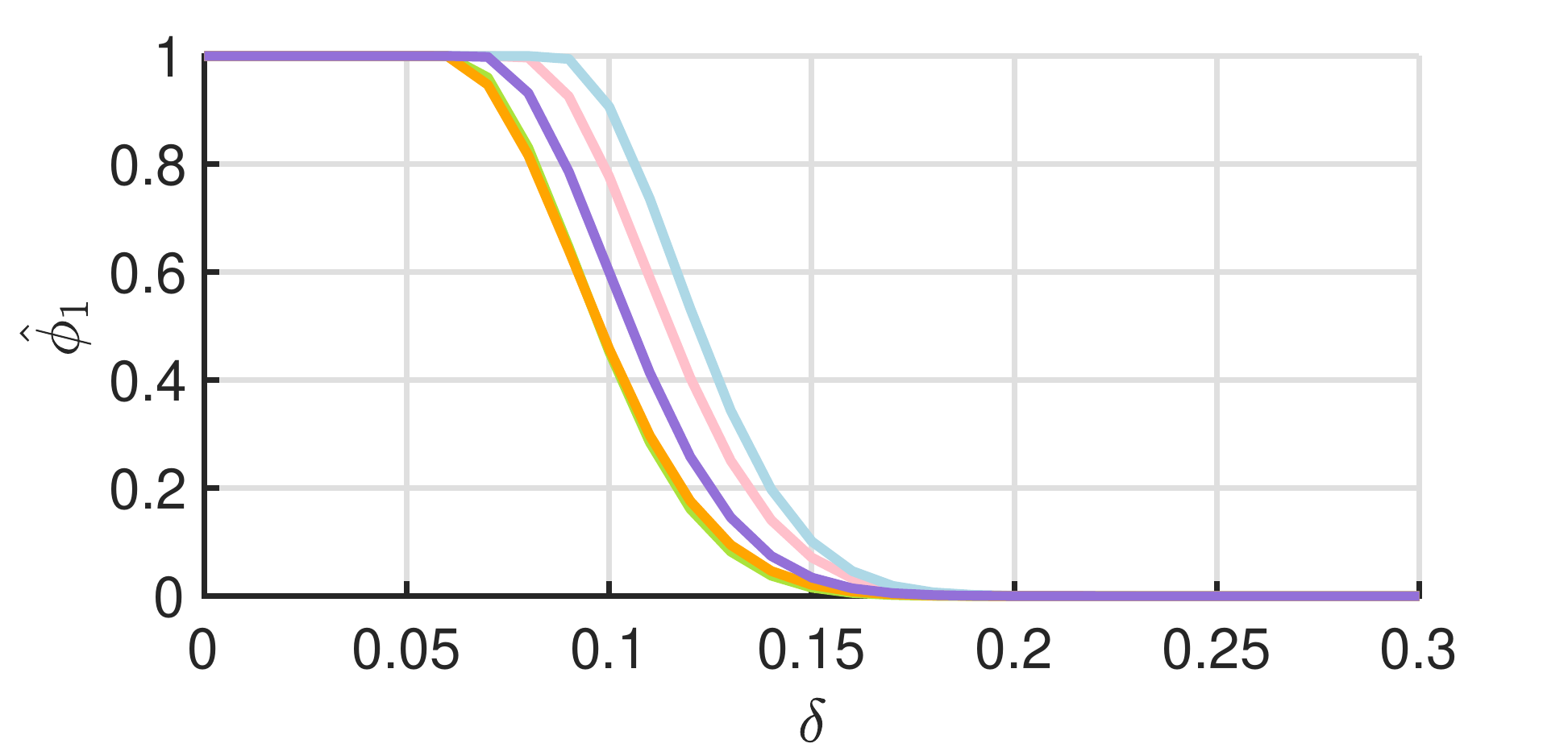}} \hfill
{\includegraphics[clip = on,trim = 0mm 0mm 15mm 0mm,width = .33\textwidth]{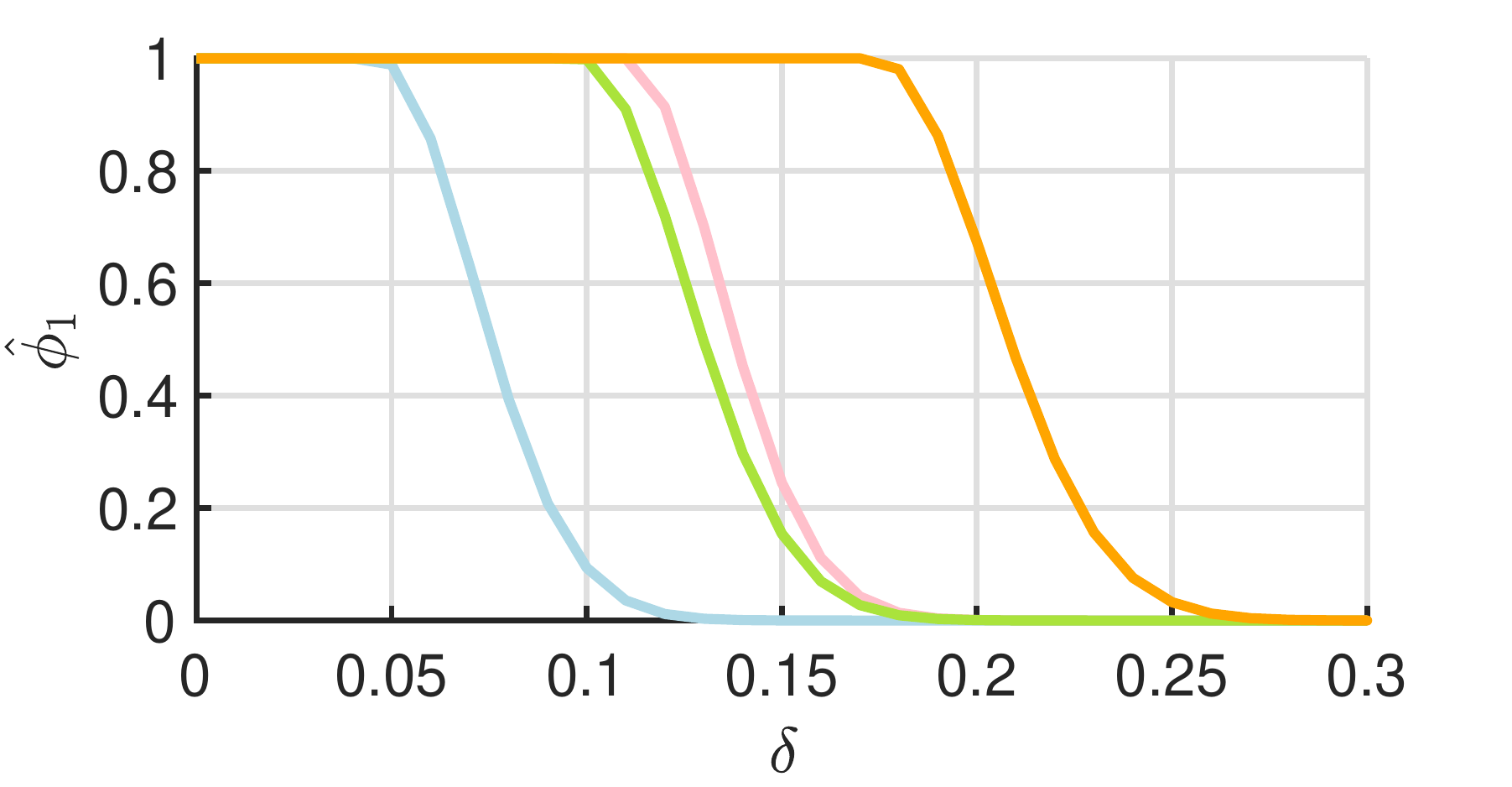}}
\caption{First row: three images randomly selected from the MNIST test set, along with detected SIFT features. Second row: respective $\hat{\phi}_1$ values for $\gamma = 0.05$. Third row: respective $\hat{\phi}_1$ values for $\gamma = 0.15$.
}
\label{fig:feat_analysis}
\end{figure*}

In this section we apply the methods presented above to GP defined with deep kernels, in an effort to provide a probabilistic 
analysis of adversarial examples.
This analysis is exact for GPs, but only approximate for fully-connected NNs, 
by virtue of \textit{weak convergence} of the induced distributions between deep kernel GPs and deep fully-connected NNs.
%
\subsection{Experimental Setting}
We focus on GPs with ReLU kernel, which directly correspond to fully-connected NNs with ReLU activation functions.
Given the number of layers $L$, the regularization parameters $\sigma_w$ (prior variance on the weights) and $\sigma_b$ (prior variance on the bias), the ReLU covariance  $\Sigma^L \left(x_1,x_2\right)$ between two input points is iteratively defined by the set of equations \cite{lee2017deep}:
\begin{align}
    \nonumber
     \Sigma^l(x_1,x_2) &=  \sigma^2_b+\frac{\sigma^2_w}{2 \pi} \sqrt{\Sigma^{l-1}(x_1,x_1)\Sigma^{l-1}(x_2,x_2)}  \nonumber \\
    & \left( \sin \beta^{l-1}_{x_1,x_2} + (\pi - \beta^{l-1}_{x_1,x_2}) \cos \beta^{l-1}_{x_1,x_2} \right) \label{KernelSquaredExponential} \\
     \beta^{l}_{x_1,x_2}&=\cos^{-1}\left( \frac{\Sigma^l(x_1,x_2)}{\sqrt{\Sigma^l(x_1,x_1)\Sigma^l(x_2,x_2)}}  \right) \nonumber
\end{align}
for $l = 1,\ldots,L$, where $\Sigma^0(x_1,x_2) =  \sigma^2_b+\frac{\sigma^2_w}{m} x_1 \cdot x_2$.
%
\paragraph{Training} We follow the experimental setting of \cite{lee2017deep}, that is, we train a selection of ReLU GPs on a subset of the MNIST dataset using least-square classification (i.e. posing a regression problem to solve the classification task) and rely on optimal hyper-parameter values estimated in the latter work.
Note that the methods we presented are not constrained to specific kernels or classification models, and can be generalized by suitable modifications to the constant computation part. 
Classification accuracy obtained on the full MNIST test set varied between $77\%$ (by training only on 100 samples) to $95\%$ (training on 2000 samples).
Unless otherwise stated, we perform analysis on the best model obtained using $1000$ training samples, that is, a two-hidden-layer architecture with $\sigma_w^2 = 3.19$ and $\sigma_b^2 = 0.00$.
\paragraph{Analysis}
For scalability purposes we adopt the idea from 
\cite{wicker2018feature,ruan2018reachability} of performing a feature-level analysis.
Namely, we pre-process each image using SIFT \cite{lowe2004distinctive}.
From its output, we keep salient points and their relative magnitude, which we use to extract relevant patches from each image, in the following referred to as \textit{features}. 
We apply the analysis to thus extracted features. 
Unless otherwise stated, feature numbering follows the descending order of magnitude. 
\begin{figure*}
    \centering
    {\includegraphics[width = .3\textwidth]{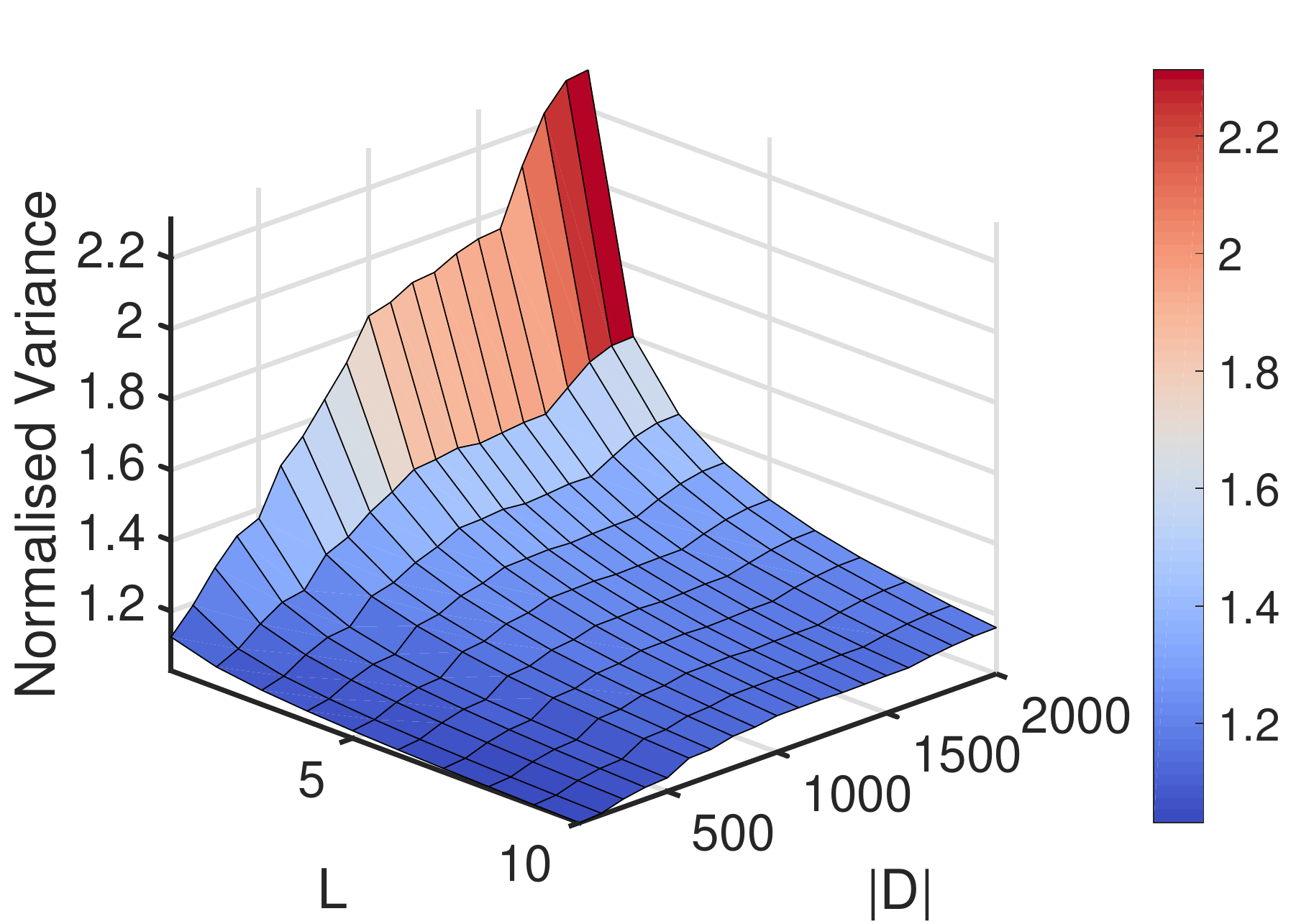}} \hfill
    {\includegraphics[width = .3\textwidth]{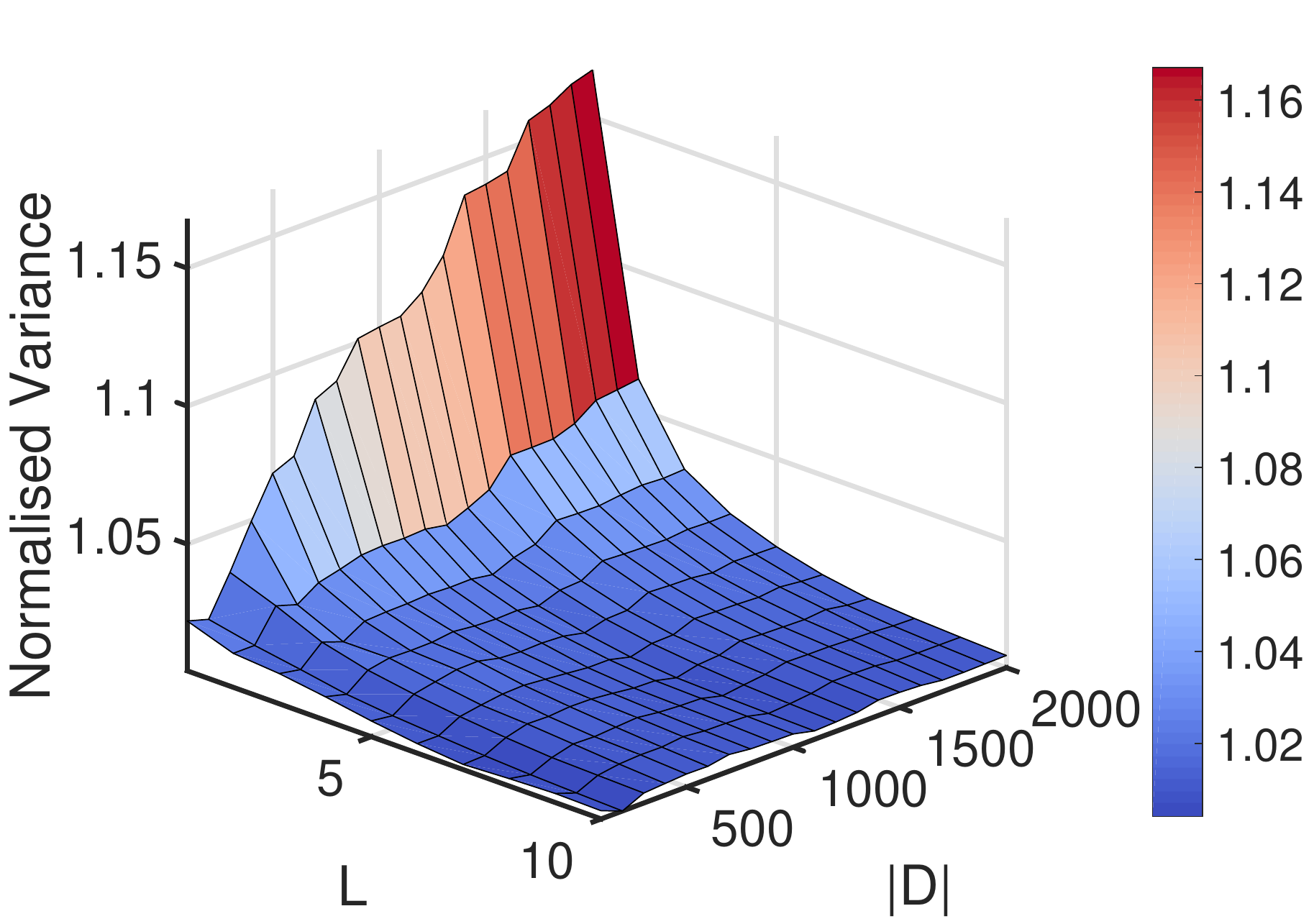}} \hfill
    {\includegraphics[width = .3\textwidth]{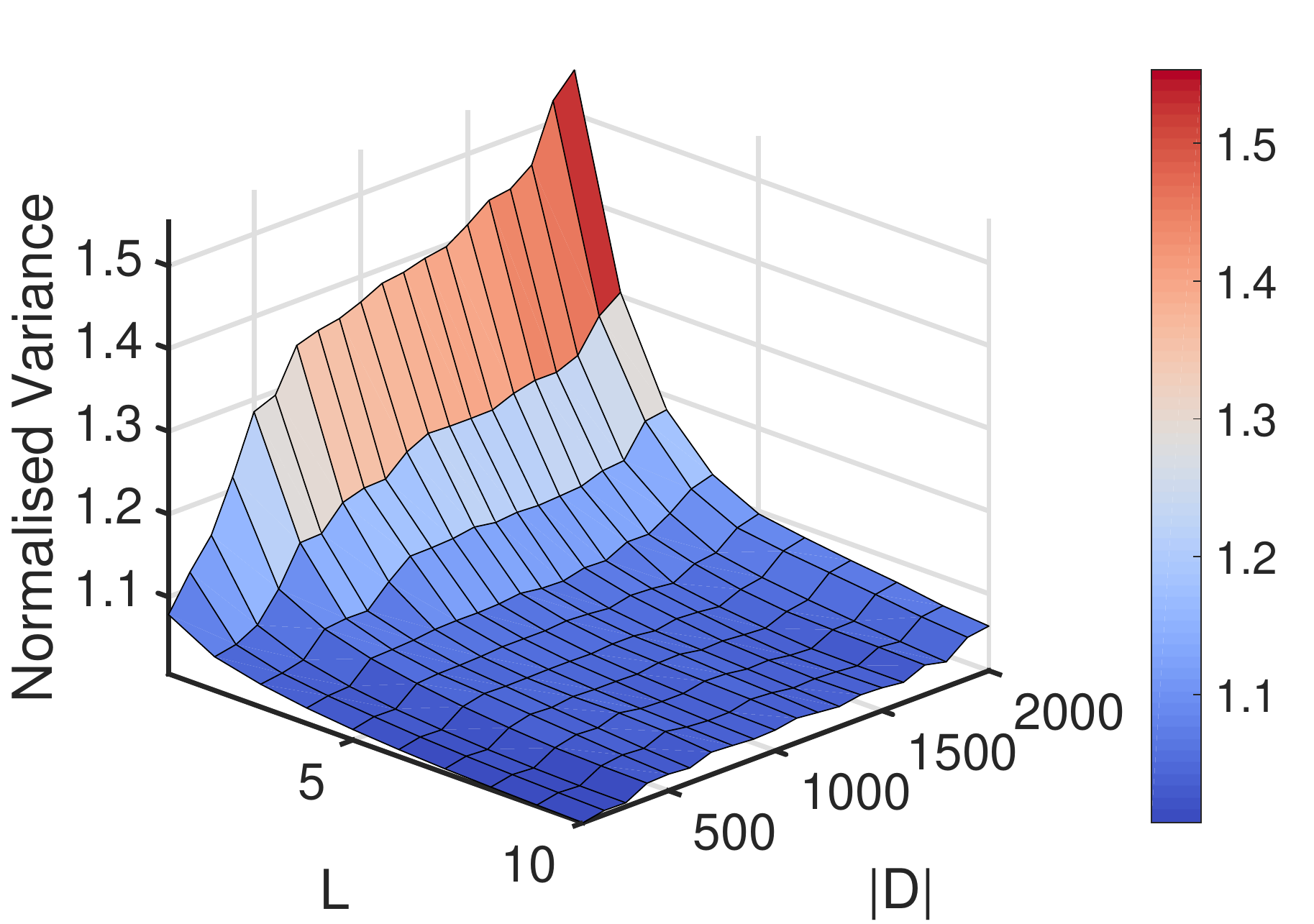}}
    \caption{Normalized variance $\bar{\sigma}^2$ as a function of $L$ (number of layers of the corresponding NN) and $\vert D \vert$ (number of training point).}
    \label{fig:VarianceNN}
\end{figure*}
\subsection{Feature-based Analysis}
In the first row of Figure \ref{fig:feat_analysis} we consider three images from the MNIST test data, and for each we highlight the first five features extracted by SIFT (or less if SIFT detected less than five features).
For each image $x_i$, feature $\mathrm{f}_j$ and $\gamma >0$ we consider the set of images $T^{\mathrm{f}_j,\gamma}_{x_i}$ given by the images differing from $x_i$ in only the pixels included in $\mathrm{f}_j$ and by no more than $\gamma$ for each pixel.

We plot the values obtained for $\hat{\phi}_1$ as a function of $\delta$ for $\gamma = 0.05$ and $\gamma = 0.15$, respectively, on the second and third row of Figure \ref{fig:feat_analysis}.
Recall that $\hat{\phi}_1$ represents an upper-bound on the probability of finding $x \in T^{\mathrm{f}_j,\gamma}_{x_i}$ such that the classification confidence for the correct class in $x$ drops by more than $\delta$ compared to that of $x_i$.
Since a greater $\gamma$ value implies a larger neighborhood $T^{\mathrm{f}_j,\gamma}_{x_i}$, intuitively $\hat{\phi}_1$ will monotonically increase along with the value of $\gamma$.  
Interestingly, the rate of increase is significantly different for different features.
In fact, while most of the 14 features 
analyzed in Figure \ref{fig:feat_analysis} have similar $\hat{\phi}_1$ values for $\gamma = 0.05$, the values computed for 
some of the features using $\gamma = 0.15 $ are almost double (e.g.\ feature 4 for the third image), and remains fairly similar for others (e.g.\ feature 3 for the first image).
Also the relative ordering in robustness for different features is not consistent for different values of $\gamma$ (e.g.\ features 2 and 5 from the first image).
This highlights the need of performing parametric analysis of adversarial attacks, which take into account different strengths and misclassification thresholds, as suggested in \cite{biggio2017wild}.
Finally, notice that, though only 14 features are explored here, the experiment shows no clear relationship between feature magnitude as estimated by SIFT and feature robustness, which calls for caution in adversarial attacks and defences that rely on feature importance.  {Note also that an empirical analysis of the robustness based on sampling, as performed in Figure \ref{fig:phi_all_run_ex}, becomes infeasible for this example as, in order to have good accuracy, a fine grid over a high-dimensional input space  would be required.  }

\subsection{Variance Analysis}
Most active defences are based upon rejecting input samples characterized by high uncertainty values.
After uncertainty is estimated, defences of this type usually proceed by setting a meta-learning problem whose goal is to distinguish between low and high variance input points, so as to flag potential adversarial examples \cite{grosse2017wrong,feinman2017detecting}.
However, mixed results are 
obtained with this approach \cite{carlini2017adversarial}.

In this subsection we aim at analyzing how the variance around test samples changes with different training settings for the three test points previously discussed.
We use the method developed for variance optimisation to compute:
$$
\bar{\sigma}^2 (x^*) = \frac{1}{ \bar{\Sigma}_{x^*,x^*}  } \sup_{x \in T^{\mathrm{f}_1,\gamma}_{x^*}} \bar{\Sigma}_{x,x},
$$
that is, we look for the highest variance point in the $T^{\mathrm{f}_1,\gamma}_{x^*}$ neighbourhood of $x^*$, and normalise its value with respect to the variance at $x^*$.
We use $\gamma = 0.15$ and perform the analysis only on feature 1 of each image.

Figure \ref{fig:VarianceNN} plots values of $\bar{\sigma}^2 (x^*)$ as a function of the number of layers (from 1 to 10) and samples (from 100 to 2000) included in the training set. 
Firstly, notice how maximum values of $\bar{\sigma}^2 (x^*)$ are perfectly aligned with the results of Figure \ref{fig:feat_analysis}. That is, less robust features are associated with higher values of $\bar{\sigma}^2 (x^*)$ (e.g.\ feature 1 for image 1).
This 
highlights
the relationship between the existence of adversarial examples in the neighbourhood of a point and model uncertainty.
We observe the normalised variance value to consistently monotonically increase with respect to the number of training samples used.
This suggests that, as more and more training samples are input into the training model, the latter become more confident in predicting ``natural'' test samples compared to ``artificial'' ones.
Unfortunately, as the number of layers increases, the value of $\bar{\sigma}^2 (x^*)$ decreases rapidly to a plateau.
This seems to point to the fact that defence methods based on a-posteriori variance thresholding become less 
effective with more complex neural network architectures, which could be a justification for the mixed results obtained so far using active defences. 

\section{Conclusion}
In this paper we presented a formal approach for safety analysis of Bayesian inference with Gaussian process priors with respect to adversarial examples and invariance properties. As the properties considered in this paper cannot be computed exactly for general GPs, we compute their safe over-approximations. Our bounds are based on the Borell-TIS inequality and the Dudley entropy integral, which are known to give tight bounds for the study of suprema of Gaussian processes \cite{adler2009random}.
On examples 
of regression tasks for GPs and deep neural networks, we showed how our results allow one to quantify the uncertainty associated to a given prediction, also taking into account of local perturbations of the input space.
Hence, we believe our results represent a step towards the application of Bayesian models in safety-critical applications.

\bibliographystyle{aaai}

\bibliography{main.bib}
 \newpage
\appendix

\onecolumn
\section{Supplementary Materials}
In what follows we report the supplementary material of the paper. We first report the proofs of the main results and then  give further details of the algorithmic framework we develop to compute the constants required in Theorem \ref{BoundsSIngleBoundedVariation} and  \ref{Theorem-Invariance}.
Finally, we give details for the case of the squared-exponential kernel and ReLu kernel.

\section*{Proofs}\label{Appendix:Proofs}
     
\paragraph{\textbf{Proof of Theorem \ref{BoundsSIngleBoundedVariation}}}     
     \begin{align*}
         & P( \exists x \in T\, s.t.\, \big( \GP^{(i)}(x^*)-\GP^{(i)}(x)>\delta \, | \, \mathcal{D}\big)\\ & \quad \quad \text{(By definition of supremum)}\\
         =&P\big( \sup_{x \in T}\,  \GP^{o,(i)}(x^*, x)>\delta \big)\\ & \quad \quad \text{(By linearity of GPs)}\\
         =&P\big( \sup_{x \in T}\,\hat \GP^{o,(i)}(x^*, x) + \mathbb{E}[\GP^{o,(i)}(x^*, x)]>\delta \big)\\ & \quad \quad \text{(By definition of supremum)}\\
         \leq &P\big( \sup_{x \in T}\, \hat \GP^{o,(i)}(x^*, x) >\delta- sup_{x_1 \in T}\mathbb{E}[\GP^{o,(i)}( x^*,x_1)] \big).
     \end{align*}
     where $\hat \GP^{o,(i)}(x^*, x)$ is the zero mean Gaussian process with same variance of $\GP^{o,(i)}(x^*, x).$ 
The last inequality can be bound from above using the following inequality, called Borell-TIS inequality \cite{adler2009random}.
     
\begin{theorem}{(Borell-TIS inequality)}\label{BorellInequalityTheorem}
Let $\hat \GP$ a zero-mean unidimensional Gaussian process with covariance matrix $\Sigma$. Assume $E[sup_{x\in T}\hat \GP(x)]<\infty$. Then, for any $u>\mathbb{E}[sup_{x\in T}\hat \GP(x)]$ it holds that
\begin{align}\label{BorellInequalityEquation}
P(sup_{x\in T}\hat \GP(x)>u)\leq e^{\frac{(u-\mathbb{E}[sup_{t\in T}\hat \GP(x)])^2}{2 \sigma_T^2}},
\end{align}
where $\sigma_T^2=sup_{x\in T}\Sigma (x)$.
\end{theorem}
     \noindent
     In order to use the Borell-TIS inequality we need to bound from above $\mathbb{E}[sup_{t\in T}\hat \GP(x)]$, the expectation of the supremum of $\hat \GP$. For Gaussian processes we can use the Dudley's entropy integral \cite{adler2009random}, which guarantees that
    \begin{align*}
        \mathbb{E}[\sup_{x\in T}&\, \hat \GP(x)]\leq 12 \int_{0}^{sup_{x_1,x_2 \in T} d(x_1,x_2)} \sqrt{ln(N(d,x,T))} dx,
    \end{align*} 
where $N(d,x,T)$ is the smallest number of balls of radius $x$ according to metric $d$ that completely cover $T$ (see \cite{adler2009random} for further details).
For a hyper-cube $T$ of dimension $D$, in order to compute $N(d,x,T),$ we first need to compute $N(L_2,r,T) $, the number of covering balls of diameter $r$ of $T$ under $L_2$ norm. As the largest hyper-cube contained inside a $m-$sphere of diameter $r$ has a side of length $\frac{r}{\sqrt{m}},$ we obtain 

$$ N(L_2,r,T) \leq \big( 1 + \frac{D \sqrt{m}}{r}\big)^m . $$
 \noindent
Now we know that for $x^* \in T$ 
$$sup_{x_1,x_2 \in T} d^{(i)}_{x^*}(x_2,x_1) \leq K_{\hat x}^{(i)} \vert \vert x_2-x_1 \vert \vert_2,$$
Thus, this implies that 
all the points inside a ball of radius $r=\frac{x}{K_{\hat x}^{(i)}}$ will have a distance in the d metric smaller or equal than $x$.
Thus, the number of covering balls of radius $x$ for $T$, according to pseudo-metric $d$ is upper-bounded by
$$  N(d,x,T) \leq  \big( \frac{\sqrt{m} D K_{\hat x}^{(i)}}{x}  +1 \big)^m. $$

\paragraph{\textbf{Proof of Theorem \ref{Theorem-Invariance}}}

\begin{align*}
         & P( \exists x \in T\, s.t.\, \vert \vert \GP(x^*)-\GP(x) \vert \vert_1>\delta \, | \, \mathcal{D} \big)\\ & \quad \quad \text{(By definition of supremum)}\\
         =&P\big( \sup_{x \in T}\,   \vert \vert \GP^{o}(x^*, x) \vert \vert_1>\delta  \big)\\ & \quad \quad \text{(By definition of $L_1$ norm)}\\
         = &P\big( \sup_{x \in T}\,   \sum_{i=1}^n| \GP^{o,(i)}(x^*, x)|>\delta \big)\\ & \quad \quad \text{(By closure of GPs wrt linear operations)}\\
         \leq &P\big( \sup_{x \in T}\,   \sum_{i=1}^n| \hat\GP^{o,(i)}(x^*, x)|>\delta- sup_{x_1 \in T}\vert \vert\mu^{o}(x_1, x^*) \vert \vert_1 \big)\\ & \quad \quad \text{(By the fact that $\forall i \in \{1,...,n \} | \hat \GP^{(i)}(x)-\hat \GP^{(i)}(x^*)|\geq 0 $)}\\
         \leq &P\big( \vee_{i \in \{1,...,n \}} \sup_{x \in T}\,   | \hat \GP^{o,(i)}(x^*, x)|> \frac{\delta- sup_{x_1 \in T}\vert \vert\mu^{o}(x_1, x^*) \vert \vert_1}{n} \big) \\ & \quad \quad \text{(By the union bound and symmetric properties of Gaussian distributions)}\\
         \leq &2 \sum_{i=1}^n P\big( \sup_{x \in T}\,    \hat \GP^{o,(i)}(x,x^*)>\frac{\delta- sup_{x_1 \in T}\vert \vert\mu^{o}(x_1, x^*) \vert \vert_1}{n} \big)\\.
     \end{align*}
Last term can be bounded by using the Borell-TIS inequality and Dudley's entropy integral, as shown in the proof of  Theorem \ref{BoundsSIngleBoundedVariation}.
 

\section{Constants Computation}\label{Appendix:ConstantComputation}
\subsection{Lower and Upper bound to Kernel Function}
In this subsection we describe a method for computing lower and linear approximation to the kernel function.
Namely, given $x \in T$ and $x^* \in \mathbb{R}^m$, in this we show how to compute $a_L$, $b_L$ such that:
$$
a_L + b_L \varphi_\Sigma \left( x, x^* \right) \leq \Sigma_{x,x^*}  \quad \forall x \in T.
$$
Notice that the same techniques can be used to find $a_U$ and $b_U$ coefficients of an upper-bound, simply by considering $-\Sigma_{x,x^*}$. 
Let $\varphi_\Sigma^L$ and $\varphi_\Sigma^U$ be maximum and minimum values of $\varphi_\Sigma(x,x^*)$ for $x \in T$, and consider the univariate and unidimensional function $\psi_\Sigma (\varphi_\Sigma) : [\varphi_\Sigma^L, \varphi_\Sigma^U ] \rightarrow \mathbb{R}  $.
We can then compute $a_L$ and $b_L$ by using the methods described below.
\paragraph{Case 1}
If $\psi_\Sigma$ happens to be concave function, than by definition of concave function, a lower bound is given by the line that links the points $(\varphi_\Sigma^L,\psi_\Sigma (\varphi_\Sigma^L)  )$ and $(\varphi_\Sigma^U,\psi_\Sigma (\varphi_\Sigma^U)  )$.
\paragraph{Case 2}
If on the other hand $\psi_\Sigma$ happens to be a convex function, than by definition of convex function, a valid lower bound is given by the tangent line in the middle point $(\varphi_\Sigma^L + \varphi_\Sigma^L)/2$ of the interval.
\paragraph{Case 3}
Assume now, that $\psi_\Sigma$ is concave in  $[\varphi_\Sigma^L, \varphi_\Sigma^C ]$, and convex in $ [\varphi_\Sigma^C, \varphi_\Sigma^U ]$, for a certain $\varphi_\Sigma^C \in [\varphi_\Sigma^L, \varphi_\Sigma^U ]$ (the same line of arguments can be used by reversing convexity and concavity). Let $a_L^1$, $b_L^1$ coefficients for linear lower approximation in $[\varphi_\Sigma^L, \varphi_\Sigma^C ]$ and $a_L^2$, $b_L^2$ analogous coefficients in $[\varphi_\Sigma^C, \varphi_\Sigma^U ]$ (respectively computed as for Case 1 and 2), and call $f_1$ and $f_2$ the corresponding functions.
Define $F$ to be the linear function of coefficients $a_L$ and $b_L$ that goes through the two points $(\varphi_\Sigma^L , \min(f_1(\varphi_\Sigma^L),f_2(\varphi_\Sigma^L)  )   )$ and $(\varphi_\Sigma^U, f_2(\varphi_\Sigma^U))$.
We then have that $F$ is a valid linear lower bound in $[\varphi_\Sigma^L, \varphi_\Sigma^U ]$ in fact:
\begin{enumerate}
    \item if $f_1(\varphi_\Sigma^L) = \min(f_1(\varphi_\Sigma^L),f_2(\varphi_\Sigma^L)$: in this case we have that $F(\varphi_\Sigma^L) = f_1(\varphi_\Sigma^L)  \leq f_2(\varphi_\Sigma^L) $, and $F(\varphi_\Sigma^U) = f_2(\varphi_\Sigma^U)$. Hence $F(\varphi_\Sigma) \leq f_2(\varphi_\Sigma)$ in $[\varphi_\Sigma^L,\varphi_\Sigma^U]$, in particular in $[\varphi_\Sigma^C,\varphi_\Sigma^U]$ as well. This also implies that $F(\varphi_\Sigma^C) \leq f_2(\varphi_\Sigma^C) \leq f_1(\varphi_\Sigma^C)$. On the other hand, $F(\varphi_\Sigma^L) = f_1(\varphi_\Sigma^L)$, hence $F(\varphi_\Sigma) \leq f_1(\varphi_\Sigma)$ in $[\varphi_\Sigma^l,\varphi_\Sigma^C]$. Combining these two results and for contrsuction of $f_1$ and $f_2$ we have that $F(\varphi_\Sigma) \leq \psi_\Sigma(\varphi_\Sigma)$ in $[\varphi_\Sigma^L,\varphi_\Sigma^U]$.
    \item if $f_2(\varphi_\Sigma^L) = \min(f_1(\varphi_\Sigma^L),f_2(\varphi_\Sigma^L)$: In this case we have $F = f_2$, we just have to show that $F(\varphi_\Sigma) \leq f_1(\varphi_\Sigma)$ in $[\varphi_\Sigma^L,\varphi_\Sigma^C]$. This immediately follow noticing that $f_2(\varphi_\Sigma^C) \leq f_1(\varphi_\Sigma^C)$ and $f_2(\varphi_\Sigma^L) \leq f_1(\varphi_\Sigma^L)$.
\end{enumerate}

\paragraph{Case 4}
In the general case, assuming to have a finite number of flex points, we can divide the interval $[\varphi_\Sigma^L, \varphi_\Sigma^U ]$ in subintervals in which $\psi_\Sigma$ is either convex or concave.
We can then proceed iteratively from the two left-most intervals by repeatedly applying case 3.

\subsection{Variance}
In this subsection we show how to compute lower and upper bound on $\xi^{(i)} = \sup_{x \in T} \Sigma_{x^*,x}^{o,(i,i)}$.
Though a similar approach to that used for the mean can be used to compute analytic bound on the variance, empirically a convex relaxation of the problem is more efficient.
By definition of $\Sigma_{x^*,x}^{o,(i,i)}$ and applying the GP inference equations, we have that:
\begin{align*}
    \Sigma_{x^*,x}^{o,(i,i)} = \left( \Sigma_{x^*,x^*} + \Sigma_{x,x} - 2 \Sigma_{x,x^*} \right) - ( \Sigma_{x^*,\mathcal{D}} \Sigma^{-1}_{\mathcal{D},\mathcal{D}}   \Sigma_{x^*,\mathcal{D}}^T + \Sigma_{x,\mathcal{D}} \Sigma^{-1}_{\mathcal{D},\mathcal{D}}   \Sigma_{x,\mathcal{D}}^T - 2 \Sigma_{x^*,\mathcal{D}} \Sigma^{-1}_{\mathcal{D},\mathcal{D}}   \Sigma_{x,\mathcal{D}}^T).
\end{align*}
As such, the computation of $\xi$ boils down to the computation of:
\begin{align}\label{eq:prob_xi_inf}
    \inf_{x \in T} \left(\Sigma_{x,\mathcal{D}} \Sigma^{-1}_{\mathcal{D},\mathcal{D}}   \Sigma_{x,\mathcal{D}}^T + 2 \Sigma_{x,x^*} - 2 \Sigma_{x^*,\mathcal{D}} \Sigma^{-1}_{\mathcal{D},\mathcal{D}}   \Sigma_{x,\mathcal{D}}^T \right)
\end{align}
as all the other terms involved in the optimization are constant with respect to $x$.
The approach is based on a quadratic convexification of the above problem, which can hence be solved by standard optimisation methods.
By defining the slack variable vector $r = \Sigma_{x,\mathcal{D}} = \left(\Sigma_{x,x_1}, \Sigma_{x,x_2}, \ldots ,\Sigma_{x,x_{\vert \mathcal{D} \vert}}   \right) $ of covariances between point $x$ and points included in the training set, and $r^* = \Sigma_{x,x^*}$ of covariance between $x$ and the test sample $x^*$, we can rewrite the optimization Problem \ref{eq:prob_xi_inf} as:
\begin{align*}
    &\inf_{x \in T}  \left( r \Sigma^{-1}_{\mathcal{D},\mathcal{D}} r^T + 2 r^* - 2 \Sigma_{x^*,\mathcal{D}} \Sigma^{-1}_{\mathcal{D},\mathcal{D}} r^T \right) \\
    &\textrm{subject to: } \quad r_l = \Sigma_{x,x_l} \quad l = 1,\ldots, \vert \mathcal{D} \vert \\
    & \qquad \qquad \quad \; \, r_* = \Sigma_{x,x_*}.
\end{align*}
Notice that the objective function of the problem is convex with respect to the variable vector $(x,r,r^*)$, since $\Sigma^{-1}_{\mathcal{D},\mathcal{D}}$ is symmetric and positive definite.
Notice that the constraints of the problem are still generally non-convex, but can be over-approximated using the methods presented in \cite{jones1998efficient}.
This lead to the definition of a convex problem on the variable vector $(x,r,r^*)$, that can be solved to compute lower and upper bound on $\xi^{(i)}$.
\subsection{Bounds on $\sup_{x_1,x_2 \in T} d^{(i)}_{x^*} (x_1,x_2)$}
An upper bound to $\sup_{x_1,x_2 \in T} d^{(i)}_{x^*} (x_1,x_2)$ follows directly from the computation of $\xi^{(i)}$.
By the fact that $d^{(i)}_{x^*}$ is a pseudometric, it follows that for every $x_1$ and $x_2$ in $T$:
$$
d^{(i)}_{x^*} (x_1,x_2) \leq d^{(i)}_{x^*} (x^*,x_1) + d^{(i)}_{x^*} (x^*,x_2)
$$
hence:
$$
\sup_{x_1,x_2 \in T} d^{(i)}_{x^*} (x_1,x_2) \leq \sup_{x_1,x_2 \in T} \left( d^{(i)}_{x^*} (x^*,x_1) + d^{(i)}_{x^*}  (x^*,x_2)\right) =  \sup_{x_1 \in T}   d^{(i)}_{x^*} (x^*,x_1) +  \sup_{x_2 \in T}   d^{(i)}_{x^*} (x^*,x_2) =  2 \xi^{(i)}.
$$

\subsection{Bounds on $K^{(i)}_{x^*}$}
In this subsection we describe how to over approximate $K^{(i)}_{x^*}$. 
Recall that for $x^*,x_1,x_2 \in \mathbb{R}^m$ we  work with the pseudo-norm $d^{(i)}_{x^*}(x_1,x_2)$ defined as in Eqn \eqref{Eqn:dNormdefinition} and for $i\in \{1,...,m \}$ need to find a constant $K^{(i)}$ such that
$$ d^{(i)}_{x^*}(x_1,x_2)\leq K^{(i)}_{x^*} ||x_1-x_2||_2.$$
In order to simplify our task we can derive over approximations of $K^{(i)}$ by working only with the priors distributions. In fact, it is easy to show that
 \begin{align*}
    d^{(i)}(x_1,x_2)=&\sqrt{ \Sigma_{x_1,x_1}^{o,(i,i)} + \Sigma_{x_2,x_2}^{o,(i,i)} - 2 \Sigma_{x_1,x_2}^{o,(i,i)}   }\\=
    &\sqrt{ \Sigma_{x_1,x_1}^{(i,i)} + \Sigma_{x_2,x_2}^{{(i,i)}} - 2 \Sigma_{x_1,x_2}^{(i,i)} - (\Sigma_{x_1,\mathcal{D}}^{(i,i),T}\Sigma_{\mathcal{D}}^{-1}\Sigma_{x_1,\mathcal{D}}^{(i,i)}  + (\Sigma_{x_2,\mathcal{D}}^{(i,i)})^T\Sigma_{\mathcal{D}}^{-1}\Sigma_{x_2,\mathcal{D}}^{(i,i)}  - 2 (\Sigma_{x_1,\mathcal{D}}^{(i,i)})^T\Sigma_{\mathcal{D}}^{-1}\Sigma_{x_2,\mathcal{D}}^{(i,i)} )   }\\
    \leq& \sqrt{ \Sigma_{x_1,x_1}^{(i,i)} + \Sigma_{x_2,x_2}^{{(i,i)}} - 2 \Sigma_{x_1,x_2}^{(i,i)}},
\end{align*}
where the last inequality follows from the fact that $\Sigma_{\mathcal{D}}^{-1}$ is symmetric and positive definite. Thus, to get over-approximations of $K^{(i)}_{x^*}$, it is enough to consider consider only the prior distributions of the system. Note also that if $m=1$, then over-approximations can be simply obtained using the mean value theorem.

\section{Squared-Exponential Kernel}\label{Appendix:ExponentialKernel}
In this Section we provide constant computation details for squared-exponential kernel.
\subsection{Definition of $\varphi_\Sigma$ and $\psi_\Sigma$}
According to the squared-exponential kernel we have   $$\Sigma_{x_1,x_2}=\sigma^2 \exp \left(-\sum_{j=1}^m \theta_j (x_1^{(j)} -x_2^{(j)})^2  \right).   $$ 
By defining:
\begin{align*}
    \varphi_\Sigma(x_1,x_2) &=  \sum_{j=1}^m \theta_j (x_1^{(j)} -x_2^{(j)})^2 \\ 
    \psi_\Sigma(\varphi_\Sigma) &= \sigma^2 \exp{ \left( -  \varphi_\Sigma(x_1,x_2)   \right) }
\end{align*}
we have that $\Sigma_{x_1,x_2} = \psi_\Sigma(\varphi_\Sigma \left(x_1,x_2\right)) $, and  $\varphi_\Sigma$ and $\psi_\Sigma$ satisfy the assumptions 1 to 3 stated in the main text (Section Constant Computation).
\subsection{Computation of $K^{(i)}_{x^*}$}
Relying only on the prior,  we have 
\begin{align*}
 \frac{    d^{(i)}(x_1,x_2)}{|x_1-x_2|_2}&\leq \frac{\sqrt{2 \sigma^2(1 - exp(-\sum_{j=1}^m \theta_j (x_1^{(j)} -x_2^{(j)})^2  )) }  }{||x_1 - x_2||_2}.
\end{align*}

Without any lost of generality, we assume $\forall j \in {1,...,m}, \, 0\leq \theta_j \leq 1$
and that $x_1,x_2$ are such that for each $j \in \{1,...,m \}$ $0\leq x_1^{(j)} \leq 1, 0\leq x_2^{(j)} \leq 1. $ Then
.

\begin{align*}
 \frac{   ( d^{(i)}(x_1,x_2))^2}{|x_1-x_2|_2}&  \leq  \frac{{2 \sigma^2(1 - exp(-\sum_{j=1}^m \theta_j (x_1^{(j)} -x_2^{(j)})^2  )) }  } { { \sum_{j=1}^m \theta_j (x_1^{(j)} -x_2^{(j)})^2 }} 
\end{align*}

Now we can introduce the variable $r=\sum_{j=1}^m \theta_j (x_1^{(j)} -x_2^{(j)})^2$ and we obtain
\begin{align*}
 \frac{    (d^{(i)}(x_1,x_2))^2}{||x_1-x_2||_2}\leq& \frac{{2 \sigma^2 (1-exp(-r)) }}{r}
\end{align*}
As everything is positive and the square root is a monotonic function we obtain:
\begin{align*}
   \bar K_{\bar{x}}^{(i)}\leq \sqrt{\sup_{r \in [0, ub ]}   \frac{{2 \sigma^2 (1-exp(-r)) }}{r}}=\sqrt{2 \sigma^2}=\sqrt{2}\sigma
\end{align*}
where $ub = sup_{x_1,x_2 \in T} \sum_{j=1}^m \theta_j (x_1^{(j)} -x_2^{(j)})^2$.

\section{ReLu Kernel}\label{Appendix:ReLu}
In this Section we provide constant computation details for the ReLU kernel.
For simplicity we focus the discussion on a single layer ReLU kernel, and notice that by the recursion of the kernel definition with more than one hidden layer, the results here presented are generalisable to an arbitrary number of layers.
\subsection{Definition of $\varphi_\Sigma$ and $\psi_\Sigma$}
By following the kernel computation procedure outlined by \cite{lee2017deep}, we pre-process each input point to have norm one before inputting it into the GP.
By doing this the one-layer ReLU kernel simplifies to:
\begin{equation*}
    \Sigma_{x_1,x_2} = \sigma_b^2 +    \frac{\sigma_w^2}{2\pi} \left( \sigma_b^2 + \frac{\sigma_w^2}{m} \right)     \left( \sin{ \left( \cos^{-1}{ \frac{\sigma_b^2 + \frac{\sigma_w^2}{m} (x_1 \cdot x_2)  } {\sigma_b^2 + \frac{\sigma_w^2}{m}} } \right)} + \frac{\sigma_b^2 + \frac{\sigma_w^2}{m} (x_1 \cdot x_2)  }{\sigma_b^2 + \frac{\sigma_w^2}{m}}    \left( \pi - \cos^{-1}{\frac{\sigma_b^2 + \frac{\sigma_w^2}{m} (x_1 \cdot x_2) }{\sigma_b^2 + \frac{\sigma_w^2}{m}}} \right)         \right).
\end{equation*}
we define:
\begin{align*}
    \varphi_\Sigma (x_1 , x_2) &= k_1 + k_2 (x_1 \cdot x_2) \\
    \psi_\Sigma(\varphi_\Sigma) &= \sigma_b^2 +    \frac{\sigma_w^2}{2\pi} \left( \sigma_b^2 + \frac{\sigma_w^2}{m} \right)     \left( \sin{ \left( \cos^{-1}{\varphi_\Sigma (x_1 , x_2)} \right)} + \varphi_\Sigma (x_1 , x_2)    \left( \pi - \cos^{-1}{\varphi_\Sigma (x_1 , x_2)} \right)         \right).
\end{align*}
where:
\begin{align*}
k_1 = \frac{\sigma_b^2} {\sigma_b^2 + \frac{\sigma_w^2}{m}} \qquad \qquad \textrm{and} \qquad \qquad
k_2 = \frac{\frac{\sigma_w^2}{m}} {\sigma_b^2 + \frac{\sigma_w^2}{m}}.
\end{align*}
Due to smoothness of trigonometric functions it is easy to see that this decomposition of the kernel satisfy assumptions 1--2 stated in the main text (Section Constants Computation).
We also have that, thanks to the linearity of the dot product, for every $c_i$ and $x_i$
$$
\sum_{i = 1}^\tss c_i \varphi_\Sigma (x, x_i) =  k_1 \left( \sum_{i = 1}^\tss c_i - 1 \right) +  \varphi_\Sigma \left( x , \sum_{i = 1}^\tss c_i x_i  \right).
$$
Hence the computation of the superior defined in Assumption 3, boils down to the trivial computation of the maximum of a dot product.
\subsection{Computation of $K^{(i)}_{x^*}$}
By taking into consideration only the priors, and by paramterising the kernel using the $\alpha = \cos^{-1} x_1 \cdot x_2$, we obtain a 1-dimensional form for $d^{(i)}(x_1,x_2))^2$ from which we can directly compute an overapproximation of $K^{(i)}_{x^*}$, as outlined in the previous sections
\section{{Kernel Functions Decomposition}}\label{Appendix:decomposition}
{We provide decomposition of commonly used kernel functions that satisfy Assumptions 1,2 and 3 stated in the main text.
\\
\textbf{Rational Quadratic Kernel} defined as:
\begin{equation*}
    \Sigma_{x_1,x_2} = \sigma^2 \left( 1 + \frac{1}{2} \sum_{j = 1}^{m} \theta_j \left( x_1^{(j)} - x_2^{(j)} \right)^2   \right)^{-\alpha}
\end{equation*}
with hyper-parameters $\sigma$, $\alpha$ and $\theta_j$, for $j=1,\ldots,m$.
\\
\textbf{Linear Kernel} defined as:
\begin{equation*}
    \Sigma_{x_1,x_2} =  \sigma^2 \sum_{j = 1}^{m}(x_1^{(j)}-\theta_j)(x_2^{(j)}-\theta_j)
\end{equation*}
with hyper-parameters $\sigma$ and $\theta_j$, for $j=1,\ldots,m$.
\\
\textbf{Periodic Kernel} defined as:
\begin{equation*}
    \Sigma_{x_1,x_2} =  \sigma^2 \exp{\left( - \frac{1}{2} \sum_{j = 1}^{m}\theta_j \sin \left( p_j(x_1^{(j)}-x_2^{(j)}) \right) ^2\right)}
\end{equation*}
with hyper-parameters $\sigma$, $\theta_j$ and $p_j$ for $j=1,\ldots,m$.
\\
\textbf{Mat\'ern Kernel} for half-integers values, defined as:
\begin{equation*}
    \Sigma_{x_1,x_2} = \sigma^2 k_{p} \exp{\left(-\sqrt{\hat{k}_p\sum_{j = 1}^{m} \theta_j (x_1^{(j)} - x_2^{(j)}) }\right)} \sum_{l = 0}^{p}k_{l,p} \sqrt{  \hat{k}_p\sum_{j = 1}^{m} \theta_j (x_1^{(j)} - x_2^{(j)}) }^{p-l}
\end{equation*}
with hyper-parameters $\sigma$, $\theta_j$, for $j=1,\ldots,m$, and (integer valued) $p$; while $k_{p}$, $\hat{k}_p$ and $k_{l,p}$ are constants.

Table \ref{tab:decomposition} shows decompositions for the kernels listed above that satisfy Assumptions 1,2 and 3.
Specifically, for the periodic kernel Assumption 3 is not strictly satisfied as it is equivalent to the computation of: 
\begin{equation*}
    \sup_{x \in T} \sum_{i = 1}^\tss c_i \sum_{j = 1}^{m}\theta_j \sin(p_j(x^{(j)}-x_i^{(j)}))^2.
\end{equation*}
Each summand separately can be trivially optimized; summing together the individual optima provides a sound over-approximation of the $\sup$.
As such, the decomposition will provide formal lower and upper bounds that can be used for branch and bound, though in general those will be looser requiring an increased number of iterations in practice.
}
\begin{table}[h]
    \centering
    { 
    \begin{tabular}{l||c|c}
         Kernel              & $\psi_\Sigma (\varphi_\Sigma)$                                                  &  $\varphi_\Sigma (x_1,x_2)$                        \\ \hline \hline
         Squared Exponential &  $\sigma^2 \exp{ \left( -  \varphi_\Sigma(x_1,x_2)   \right) }$               &   $\sum_{j=1}^m \theta_j (x_1^{(j)} -x_2^{(j)})^2$  \\ \hline
         \multirow{2}{*}{ReLU}                &   $\sigma_b^2 +    \frac{\sigma_w^2}{2\pi} \left( \sigma_b^2 + \frac{\sigma_w^2}{m} \right)     \big( \sin{ \left( \cos^{-1}{\varphi_\Sigma (x_1 , x_2)} \right)} +     $ &  \multirow{2}{*}{$ k_1 + k_2 (x_1 \cdot x_2) $}       \\ 
         & $ + \varphi_\Sigma (x_1 , x_2)    \left( \pi - \cos^{-1}{\varphi_\Sigma (x_1 , x_2)} \right)         \big)$ & \\ \hline
         Rational Quadratic  & $\sigma^2(1+\frac{\varphi_\Sigma}{2})^{-\alpha} $                               & $\sum_{j = 1}^{m} \theta_j(x_1^{(j)}-x_2^{(j)})^2$ \\ \hline
         Linear              & $\sigma^2\varphi_\Sigma$                                                        &  $\sum_{j = 1}^{m}(x_1^{(j)}-\theta_j)(x_2^{(j)}-\theta_j)$ \\ \hline 
         Periodic            & $\sigma^2 \exp(-0.5\varphi_\Sigma)$                                             & $\sum_{j = 1}^{m}\theta_j \sin(p_j(x_1^{(j)}-x_2^{(j)}))^2$ \\ \hline
         Mat\'ern  &  $\sigma^2 k_{p} \exp{(-\sqrt{\varphi_\Sigma})}\sum_{l = 0}^{p}k_{l,p} \sqrt{\varphi_\Sigma^{p-l}}$  &        $\hat{k}_p\sum_{j = 1}^{m} \theta_j (x_1^{(j)} - x_2^{(j)})$ \\ \hline \hline
    \end{tabular}}
    \caption{{Kernels decomposition that satisfy the three assumptions stated in the main text (Section Constant Computation).
    Decomposition for  Mat\'ern kernel is given only for half-integer values. Decomposition for the ReLU kernel is given in the case of one-hidden layer, generalisation to an arbitrary number of layers can be obtained by recursive application of the formulas.}}
    \label{tab:decomposition}
\end{table}



\end{document}